\documentclass[review,11pt]{elsarticle}

\usepackage{lineno,hyperref}
\modulolinenumbers[5]

\usepackage[margin=3cm]{geometry}
\usepackage{setspace}

\usepackage{graphicx}
\usepackage{times,amsmath,epsfig}
\graphicspath{{./}}
\usepackage{algorithmic}
\usepackage[ruled,lined,commentsnumbered]{algorithm2e}
\usepackage{tikz}
\usetikzlibrary{shapes,arrows}
\usepackage{url}
\usepackage{stfloats}
\usepackage{fixltx2e}
\usepackage[tight,footnotesize]{subfigure}
\usepackage{eqparbox}
\usepackage{mdwmath}
\usepackage{mdwtab}
\usepackage{array}
\usepackage{booktabs}
\usepackage{multirow}
\usepackage{subfigure}
\usepackage{rotating}
\usepackage{hyperref}
\usepackage{amssymb}
\usepackage{amsthm}
\usepackage{inputenc}
\newtheorem{theorem}{Theorem}
\newtheorem{definition}{Definition}

\let\OLDthebibliography\thebibliography
\renewcommand\thebibliography[1]{
  \OLDthebibliography{#1}
  \setlength{\parskip}{0pt}
  \setlength{\itemsep}{0pt plus 0.3ex}
}

\journal{Pattern Recognition}

\begin{document}

\begin{frontmatter}

\title{Adaptive Decision Forest: An Incremental Machine Learning Framework}

\author{Md Geaur Rahman\corref{cor1}}
\ead{grahman@csu.edu.au}
\author{Md Zahidul Islam\corref{cor2}}
\ead{zislam@csu.edu.au}
\cortext[cor1]{Corresponding author. Tel: +61 4 8010 0601, \url{http://gea.bau.edu.bd}}
\cortext[cor2]{URL: \url{http://csusap.csu.edu.au/~zislam/} (Md Zahidul Islam) }
\address{School of Computing and Mathematics, Charles Sturt University, Australia}

\begin{abstract}
In this study, we present an incremental machine learning framework called Adaptive Decision Forest (ADF), which produces a decision forest to classify new records. Based on our two novel theorems, we introduce a new splitting strategy called iSAT, which allows ADF to classify new records even if they are associated with previously unseen classes. ADF is capable of identifying and handling concept drift; it, however, does not forget previously gained knowledge. Moreover, ADF is capable of handling big data if the data can be divided into batches. We evaluate ADF on five publicly available natural data sets and one synthetic data set, and compare the performance of ADF against the performance of eight state-of-the-art techniques. Our experimental results, including statistical sign test and Nemenyi test analyses, indicate a clear superiority of the proposed framework over the state-of-the-art techniques.
\end{abstract}


\end{frontmatter}


\section{Introduction}
\label{Introduction}

Nowadays information is considered the backbone of all organizations and is critical for their success. In real applications, big data often arrive as batches over time~\cite{gama2014survey}. Let us consider a scenario of an undergraduate admission system of a university where the admission authority can identify those applicants who have a high chance of completing the degree. The likelihood of success of an applicant can be determined by comparing the applicant's information (such as academic record, age, gender, nationality, etc.) with similar information from successful graduates. The yearly list of both successful and unsuccessful students can be considered as the batch training data and the yearly undergraduate admission applicants can be considered as the batch testing data, as shown in Fig.~\ref{fig:justification_test_batch}. Moreover, the batch testing data may follow not only the distribution of current batch training data but also the distributions of some previous batches of training data. In Fig.~\ref{fig:justification_test_batch} we can see that a group of students obtained their qualification (i.e. completed Year 12 degree) in the current year and another group of students obtained the qualification in previous years. Both groups of students are generally eligible to apply for the admission. The batch testing data of the year 2019 (i.e. Test batch 3) may have the students who obtained the admission qualification in the years 2017, 2018 and 2019.
\begin{figure}[ht!]
\centering
  \setlength{\belowcaptionskip}{0pt}
	\setlength{\abovecaptionskip}{0pt}	
	\includegraphics[width=0.90\linewidth]{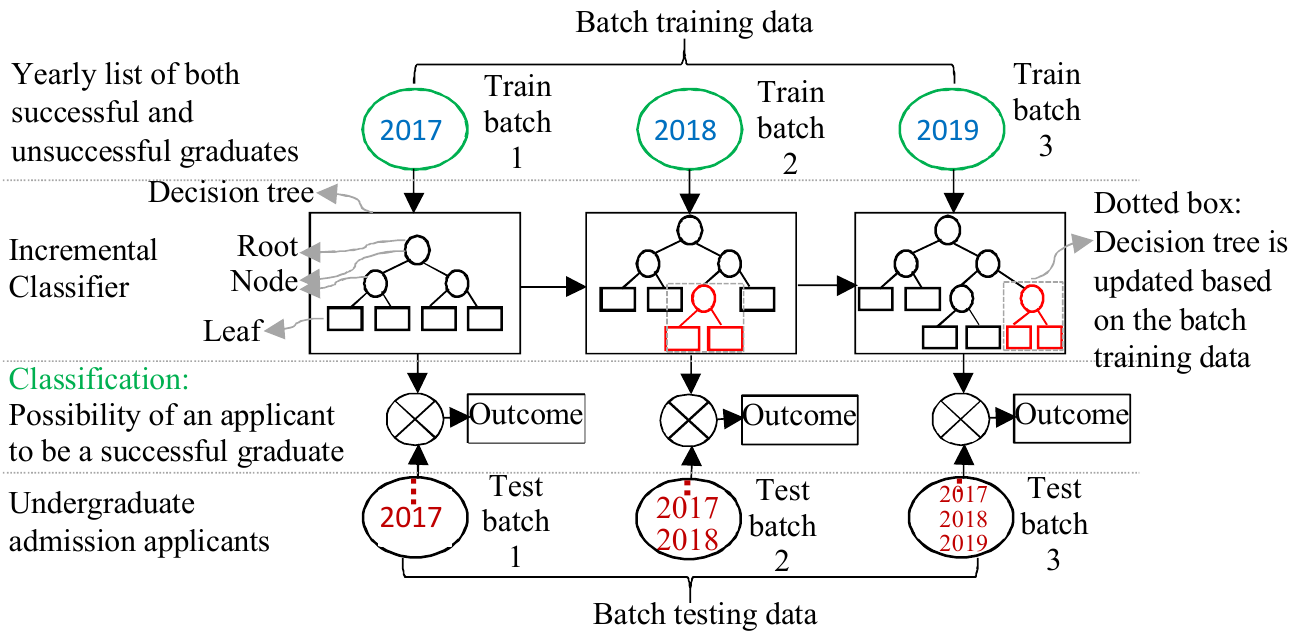}
	 \caption{An example of an undergraduate admission system of a university where both training and testing data arrive yearly as batches and the incremental classifier is updated (marked with the dotted box) each year based on the batch training data.}
	 \label{fig:justification_test_batch}
\end{figure}

Traditional machine learning algorithms such as Random Forest (RF)~\cite{breiman2001random} build recognition models based on training data, which are available at a time, to classify test data. Since all the training data in many real applications may not be available at a time, traditional machine learning algorithms are unable to build good models~\cite{hu2018novel}. Accurate recommendations and recognition are also an important challenge for the algorithms since they are not capable of adapting to dynamic changes in real applications~\cite{hu2018feature, hu2018novel}. Dynamic changes in data are also known as concept drifts. To adapt to the concept drifts, the models of traditional algorithms need to be retrained from scratch, which leads to large waste of time and memory~\cite{hu2018feature, hu2018novel}. Also, the data can be so big that it may not fit in a memory to be processed by traditional machine learning algorithms~\cite{hu2018novel}. 

Thus, to adapt concept drifts and to handle big data that often arrive as batches over time, it is crucial to have learning algorithms that are capable of learning incrementally and building a knowledge base over time for ensuring accurate classification of test batches that follow the distributions of current and previous batches~\cite{hu2018novel}. In Fig.~\ref{fig:justification_test_batch} we can see that an incremental classifier is built by the authority to assess the possibility of applicants to be successful graduates. Moreover, the classifier is updated incrementally based on the yearly batch training data.

A number of methods have been proposed recently in the literature for incremental learning~\cite{ristin2015incremental, hu2018novel, mensink2013distance, gomes2017adaptive, oza2005online}. An existing method called Nearest Class Mean Classifier (NCMC)~\cite{mensink2013distance} classifies a new record based on the centroids. It computes the centroid of the records labeled with a class and for a new unlabeled record it assigns the class value for the centroid of which is the closest to the record among all centroids. It updates the centroids based on the records that arrive over time. Although the method requires a low execution time for adapting the model with new records, it suffers with a low classification accuracy if the initial training data set is small~\cite{ristin2015incremental, ristin2014incremental}. 

The classification accuracy is increased significantly in a recent technique called CIRF~\cite{hu2018novel}, which is capable of updating an incremental classifier based on the current batch data only. CIRF first builds a decision forest by applying the RF~\cite{breiman2001random} on the initial batch data. It also identifies the boundary of the decision forest by considering the records as a box, which is also known as the Axis Aligned Minimum Bounding Box (AABB)~\cite{arvo1990transforming} (see Definition~\ref{defaabb}). A sample AABB of a set of data points is illustrated in Fig.~\ref{fig:aabb}a. Using the AABB, CIRF calculates two vectors called minimum and maximum vectors, where the $j$th element of the minimum vector is the minimum value of the $j$the attribute and the $j$th element of the maximum vector is the maximum value of the $j$the attribute.
\begin{figure}[ht!]
\centering
  \setlength{\belowcaptionskip}{0pt}
	\setlength{\abovecaptionskip}{0pt}	
	\includegraphics[width=0.95\linewidth]{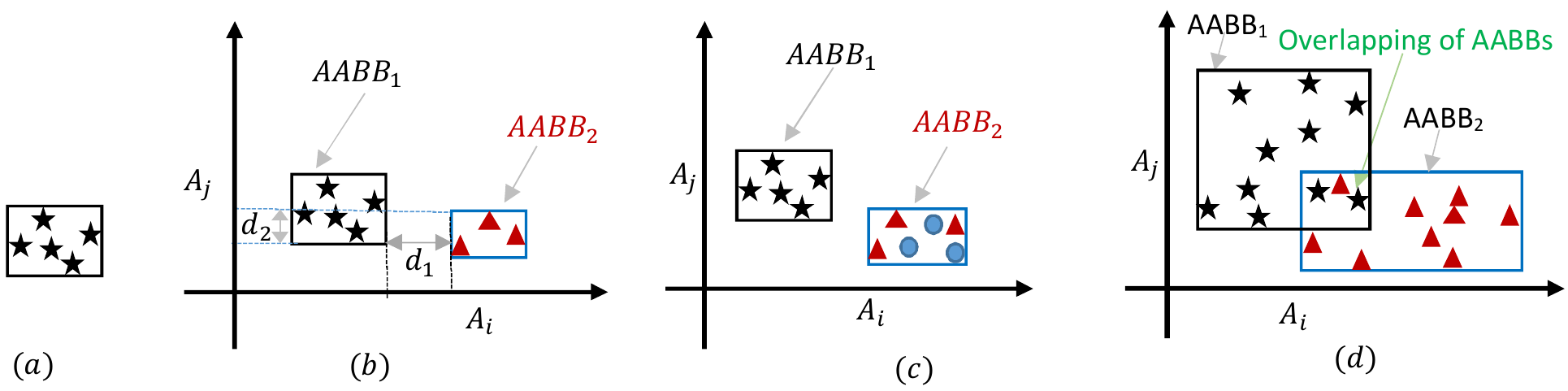}
	 \caption{(a) An Axis Aligned Minimum Bounding Box (AABB) of a set of data points represented by black stars. (b) Two AABBs are illustrated in a 2-dimensional form where attributes $A_i$ and $A_j$ are projected on x-axis and y-axis, respectively. (c) ${AABB}_2$ contains records having multiple class values. (d) There is an overlap between ${AABB}_1$ and ${AABB}_2$ due to the mixture of known (black stars) and unknown (red triangles) classes.}
	 \label{fig:aabb}
\end{figure}

For a new batch data having known class values\footnote{Known class values: the class values that appeared in the previous batches and hence previous forests were built from the batches having these class values.}, CIRF updates a decision tree of the decision forest just by assigning the records into leaves they belong to. However, for a new batch data having unknown class values\footnote{Unknown class values: the class values that do not appear in the previous batches and hence previous forests were built from the batches not having these class values.}, CIRF updates a decision tree by adding a root where the previous tree is considered as a child of the root and the new batch data are used to create another child of the root. CIRF uses an existing splitting strategy called Separating Axis Theorem (SAT)~\cite{gottschalk1996separating} (see Theorem~\ref{sat_theorem}) to determine the split attribute as follows. Let ${AABB}_1$ and ${AABB}_2$ be two boxes that are represented by the records of previous batches and the records of current batch, respectively (as shown in Fig.~\ref{fig:aabb}b). CIRF calculates the minimum and maximum vectors for both ${AABB}_1$ and ${AABB}_2$. Using the vectors, the method then calculates the distance between ${AABB}_1$ and ${AABB}_2$ for all numerical attributes. For example, in Fig.~\ref{fig:aabb}b $d_1$ is the distance between AABB1 and AABB2 for the attribute $A_i$ which is projected on the x-axis. The attribute which has the maximum distance is chosen as the split attribute (see Eq.~\eqref{sat_splitAttr_eq} for details). It is reported that CIRF outperforms state-of-the-art incremental learning algorithms in terms of both classification accuracy and execution time~\cite{hu2018novel}.  

We argue that CIRF has three main shortcomings. First, CIRF considers only the numerical attributes while calculating the minimum and maximum vectors and thereby, categorical attributes are never considered as the split attributes. Second, for a new batch having multiple unknown class values (such as ${AABB}_2$ Fig.~\ref{fig:aabb}c) CIRF adapts the decision tree by creating new a leaf under the root of the tree. Due to the heterogeneity in the new leaf, CIRF leads to a low classification accuracy. Third, if a new batch contains records with a mix of known and unknown classes (such as ${AABB}_2$ Fig.~\ref{fig:aabb}d), there will be an overlap between the AABBs (as shown in Fig~\ref{fig:aabb}d) resulting is a high heterogeneity in the newly created leaf. Thus, CIRF suffers with a low classification accuracy (see Fig.~\ref{fig:justify_sat_entropy_isat}).     

The existing methods therefore have room for further improvement. We propose a novel incremental learning framework called Adaptive Decision Forest (ADF) in this paper. In the ADF framework, we build decision forests by using one of the three techniques: RF~\cite{breiman2001random}, HT~\cite{domingos2000mining} and SysFor~\cite{islam2011knowledge}, and thereby obtain three variants called ADF-R, ADF-H and ADF-S, respectively. 
The ADF framework is adapted for the training batch data that arrive over time. During adaptation, ADF makes use of our proposed novel splitting strategy called improved Separating Axis Theorem (iSAT) to find the best split attribute. In iSAT, we use both SAT and entropy-based strategies. For a new batch having a mix of known and unknown classes, ADF can find two sub-AABBs (as shown in Fig~\ref{fig:aabb}d and Fig.~\ref{fig:sub_aabb_overlap_proof}) based on our proposed theorems (see Theorem~\ref{subbox_nonoverlap_theorem} and Theorem~\ref{subbox_overlap_theorem}) where a sub-AABB contains records with a minimum number of a mix of known and unknown classes and the other sub-AABB has records with new class values. For the first sub-AABB, we use the entropy based splitting strategy and for the second sub-AABB, we consider SAT based splitting strategy. 

The main contributions of this paper can be summarized as follows. 1) We present a novel splitting strategy called improved separating axis theorem (iSAT), which makes use of the separating axis theorem (SAT)~\cite{gottschalk1996separating} and our proposed two theorems (Theorem~\ref{subbox_nonoverlap_theorem} and Theorem~\ref{subbox_overlap_theorem}) to find the best split attribute which can be either numerical or categorical. 2) We present a framework called Adaptive Decision Forest (ADF) which can identify and handle concept drifts and preserve previously acquired knowledge by introducing a set of forests (see Section~\ref{justify_pf_af_tf}). 3) ADF can handle all scenarios (as shown in Fig.~\ref{fig:framework}) that may occur for the new batches that arrive over time. 4) ADF is also applicable to big data applications where the data can be divided into batches (see section~\ref{Experimental_bigdata}).

We evaluate ADF variants on five real data sets~\cite{Frank+Asuncion:2010} and one synthetic data set by comparing its performance with the performance of eight high quality existing techniques including HT~\cite{domingos2000mining}, CIRF~\cite{hu2018novel}, and ARF~\cite{gomes2017adaptive}. We also compare the performance of ADF variants with the performance of two non-incremental existing algorithms such as SysFor~\cite{islam2011knowledge} and RF~\cite{breiman2001random}. Our experimental results indicate that ADF variants achieve a higher classification accuracy than the existing methods.

The rest of the paper is organized as follows. Section~\ref{problemAndRelatedWork} presents the problem formulation and assumptions and a background study on incremental learning methods. Our proposed incremental framework is presented in Section~\ref{Our_Framework}. Section~\ref{Experimental_Result} presents empirical evaluations, and Section~\ref{Conclusion} gives concluding remarks.

\section{Problem Formulation and Related Work}
\label{problemAndRelatedWork}

\subsection{Problem Formulation and Assumptions}
\label{problem_definition}

Let $D$ be an input data set which we consider as a two dimensional table, where rows represent records $X= \{X_1, X_2, \ldots X_n\}$ and columns represent attributes $A=\{A_1, A_2, \ldots A_m\}$. There are $m$ attributes i.e. $|A|= m$ and $n$ records i.e. $|D| = |X| = n$. $X_{ij}$ represents the $j$th attribute value of the $i$th record. An attribute $A_j\in A$ can be categorical or numerical. The domain of a categorical attribute $A_j$ can be $\{a_j^1, a_j^2, \ldots a_j^k\}$ meaning the domain size of $A_j$ is $|A_j| = k$. Similarly, the domain of a numerical attribute $A_p$ can be $[a_p^{low}, a_p^{up}]$, where $a_p^{low}$ and $a_p^{up}$ are the lower and upper limits of the domain, respectively. Let $Y\in A$ be the class attribute having a set of class values $C= \{C_1, C_2, \ldots C_k\}$. A classifier $T$ is a function $T\leftarrow f(X):D\rightarrow Y$ that maps records $X\in D$ to the class values. The notations that we use in this paper are presented in Table~\ref{tab:notations}.

\begin{table}[ht!]
	\footnotesize
	\centering
	\caption{Notations and their meanings.}
	\renewcommand\tabcolsep{3pt} 
	\begin{tabular}{ll|ll}
	\toprule
		Notations&Meaning&Notations&Meaning\\
		\midrule
		  $D$&Data set& $P(X)$& Probability of $X$\\
		  $D_{train}$& Training data set& $P(Y|X)$& Probability of $Y$ given $X$\\
			$D_{test}$& Test data set& $P(Y|\neg X)$& Probability of $Y$ with changes $X$\\
			$D^{t^i}_B$& Batch data set at time $t^i$& $P(X,Y)$& Probability distribution of $X$ and $Y$\\	
			${D^{t^i}_{B_{train}}}$& Training batch data set at time $t^i$& $\Psi$ & Set of parameters\\
			$D^{t^i}_{B_{test}}$& Test batch data set at time $t^i$& $\theta$& Repairable threshold\\
			$X$ & Set of records in $D$& $\lambda$& Concept drift threshold\\
			$X_i$ & $i$-th record in $D$& $\gamma$& Reserve window threshold\\
			$|D| = |X| = n$& Number of records in $D$& $T$& Classifier/decision forest\\
			$|{D^{t^i}_{B_{train}}}| = n_B$& Number of records in ${D^{t^i}_{B_{train}}}$& $l$& Number of leaves of a tree\\
			$A$ & Set of attributes in $D$& $T^A$ & Active Forest (AF) \\
			$|A|=m$ & Number of attributes in $D$& $T^P$ & Permanent Forest (PF)\\
			$A_j$ & $j$-th attribute in $D$& $T^T$ & Temporarily Forest (TF)\\
			$|A_j|$ & Domain size of $j$-th attribute in $D$& $L^A$ & Leaves statistics of $T^A$\\
			$a_j^{low}$ & Lower limit of $j$-th numerical attribute& $L^P$ & Leaves statistics of $T^P$\\
			$a_j^{up}$ & Upper limit of $j$-th numerical attribute& $L^T$ & Leaves statistics of $T^T$\\
			$Y$& Class attribute in $D$& $D^W$ & Reserve window data set \\
			$C$& Set of class values in $D$& $P^A$ & Perturbed leaves of $T^A$\\
			$C_k$& $k$-th class value in $C$& $P^U$ & Perturbed leaves of $T^P$\\
			$M$ & Number of trees & $P^R$ & Perturbed leaves of $T^T$\\
			$\epsilon$ & Error tolerance threshold& & \\
	\bottomrule
	\end{tabular}
	\label{tab:notations}
\end{table}

Let $D_{train}=\{X_i,Y_i\}, X_i\in X, Y_i\in C$ and $D_{test}=\{X^{\prime}_i\}, X^{\prime}_i\in X^{\prime} \: {X\cap}X^{\prime}\rightarrow \emptyset$ be the training data set and test data set, respectively, where the test data set does not have the class attribute. Traditional supervised learning algorithms assume that the whole training data set $D_{train}$ is available during training, and both $D_{train}$ and $D_{test}$ follow the same probability distribution $P(X,Y)=P(Y|X).P(X)$. The goal of the supervised learning algorithms is to build a classifier $T$, from $D_{train}$, which is capable of classifying the records of the test data set $D_{test}$ with a high classification accuracy. 

Sometimes the scenario can be more complex where the whole input data set $D_{train}$ is not available during building the classifier $T$. Sometimes although the whole data is available, it may not fit in a single memory due to its size being too big. Many incremental learning algorithms assume that data arrive as batches over time~\cite{hu2018novel, ristin2015incremental}. Moreover, the distribution $P(X,Y)=P(Y|X).P(X)$ of a training data set may change over time and may cause concept drift i.e. $P(Y|\neg X)$ which may have a serious impact on classification accuracy. However, the test data sets may follow both current and previous distributions. Let  $D^{t^0}_{B_{train}}$, $D^{t^1}_{B_{train}}$ and $D^{t^2}_{B_{train}}$ be three batches of training data with the class attribute arrive at time $t^0$, $t^1$, and $t^2$, respectively and $D^{t^0}_{B_{test}}$, $D^{t^1}_{B_{test}}$ and $D^{t^2}_{B_{test}}$ be three corresponding test batches at time $t^0$, $t^1$, and $t^2$, respectively. The test batch $D^{t^2}_{B_{test}}$ contains data that may follow the distributions of $D^{t^2}_{B_{train}}$, $D^{t^1}_{B_{train}}$ and $D^{t^0}_{B_{train}}$ (as discussed in Section~\ref{Introduction} and demonstrated in Fig~\ref{fig:justification_test_batch}). Thus, it is essential to adapt the classifier over time to achieve a high classification accuracy. Our proposed incremental framework is formulated under the following assumptions: 1) Data arrive as batches over time. Let $D^{t^i}_{B_{train}}$ and $D^{t^j}_{B_{train}}$ be the batches of data arriving at time $t^i$ and $t^j$, respectively. 2) The set of class values $C$ may change over time. Let $C^{t^i}$ and $C^{t^j}$ be the sets of class values of $t^i$ and $t^j$, respectively, where $C^{t^i}$ and $C^{t^j}$ may differ. 3) The class attribute ${Y}^i$ in the training batch data set ${D^{t^i}_{B_{train}}}$ at $t^i$ is available, thus, it is possible to adapt the classifier $T^i$ over time. 4) $D^{t^i}_{B_{test}}$ follows the distributions of training batch data sets arriving at times ${\{t^j\}}^{\gamma}_{j=i}$.

Under such assumptions, the goal of our proposed framework is to define an incremental classifier $T^i\leftarrow f^i(X^i|T^{i-1},D^{t^i}_{B_{train}}, D^{t^i}_{B_{test}}, \Psi)$ based on a decision forest which will achieve an accurate classification for the test batch data set $D^{t^i}_{B_{test}}$ at time $t^i$ by adapting the classifier $T^{i-1}$ at time $t^{i-1}$ with the training batch data set $D^{t^i}_{B_{train}}$ and the set of parameters $\Psi$.

\subsection{Related Work}
\label{Related_Work}

A number of methods have been proposed for incremental learning~\cite{ristin2015incremental, hu2018novel, mensink2013distance}. Incremental learning methods process records one by one or batch by batch that arrive over a period of time~\cite{gama2014survey}. An important challenge of the methods is to handle dynamic changes that may occur in real applications. Generally dynamic changes can happen in two ways: 1) changes in data distribution $\&$ feature dimensions, and 2) changes in class values~\cite{hu2018novel}. 

For handling the changes in data distribution $\&$ feature dimensions, a number of feature incremental learning algorithms~\cite{hu2018feature, liu2008incremental} have been proposed. An existing method called FLSSVM~\cite{liu2008incremental} that makes use of a least square support vector machine algorithm to adapt a classifier with the new batch data set. For the initial batch data set, the method builds a classifier using an existing least square support vector machine algorithm~\cite{suykens1999least}. With the previously learned structural parameters, FLSSVM then adapts the data of the new attributes. Although FLSSVM requires comparatively low training time and memory, it is difficult to find a suitable kernel function for SVM for achieving a good accuracy~\cite{suykens1999least}. 

For handling the changes in class values a number of class incremental learning algorithms~\cite{ristin2015incremental, hu2018novel, he2011incremental} have been proposed. An existing class incremental learning method called NCMC~\cite{mensink2013distance} adapts the new batch records $X_j\in D^{t^j}_B$ based on the mean vectors $V_{C_k} \forall k$ of class values $C$. For each class value $C_k\in C$ of the initial data set, NCMC first calculates the mean vector $V_{C_k}$ of the records $X_{C_k}$ belonging to the class $C_k\in C$ as follows.
\begin{align}
  V_{C_k}={\frac{1}{|X_{C_k}|}}{\sum_{X_i\in X_{C_k}} {X_i}}
\label{ncmc_mean}	
\end{align}

For a new record $X_j\in D^{t^j}_B$ the method then calculates distances separately between the new record and the means of the class values. The Euclidean distance $d^j_k$ between $X_j\in D^{t^j}_B$ and mean vector $V_{C_k}; C_k\in C$ is calculated as follows.

\begin{align}
  d^j_k=\left\|X_j-V_{C_k}\right\|_2
\label{ncmc_distance}	
\end{align}

The record $X_j\in D^{t^j}_B$ is classified by assigning the class value $C_k\in C$ which has the minimum distance. The method then recalculates $V_{C_k} \forall k$ based on the records that arrive over time. Although the method requires a low execution time for adapting the model with new records, it suffers with a low classification accuracy if the initial training data set does not have enough records~\cite{ristin2015incremental, ristin2014incremental}.

This problem is addressed in an existing method called NCMF~\cite{ristin2014incremental} which we discussed in the Introduction section. NCMF first builds a decision forest $T$ by applying an existing decision forest algorithm called Random Forest (RF)~\cite{breiman2001random} on the initial batch data $D^{t^i}_B$. The decision forest $T$ is then updated for the batch data sets that arrive over time. For a batch data set $D^{t^j}_B$, NCMF updates the decision forest $T$ as follows. NCMF first assigns the records $X_j\in D^{t^j}_B$ in the leaves of a tree $t\in T$ the records belong to. If a leaf $l$ contains records having multiple class values, NCMF calculates the mean vectors $V_{|C|}$ by using the Eq.~\ref{ncmc_mean}. NCMF then translates the mean vectors into two class values $e\in\{+ve,-ve\}$ by randomly assigning the mean vectors either into Positive or Negative. The splitting function $f(X_j)$ for assigning a record $X_j\in D^{t^j}_B$ is defined as follows~\cite{ristin2014incremental}. 
\begin{align}
  f(X_j)= e_{C_k^*(X_j)} ~~~~~~~~  where~~~ C_k^*(X_j)=  \underset{C_K\in C}{argmin}    \left\|X_j-V_{C_k}\right\|_2
\label{ncmf_splitfunction}	
\end{align}
Using the splitting function $f(X_j)$, NCMF splits the leaf $l$ into two sets where one set is considered as the left child and the other set is considered as the right child. It continues this splitting process if the size of any child is greater than a user-defined threshold. It is reported that NCMF achieves a high accuracy over some existing methods including NCMC and RF. However, during the splitting process the method requires all previous data, resulting is a high computational time~\cite{ristin2015incremental}.

The computational time is reduced in a recent technique called CIRF~\cite{hu2018novel} that first builds a decision forest $T$ by applying the RF on the initial batch data $D^{t^i}_B$. The decision forest $T$ is then updated by using only the records of the current batch data set $D^{t^j}_B$. The procedure of updating the decision forest $T$ is discussed in detail in the Introduction section. Although CIRF outperforms state-of-the-art incremental learning algorithms in terms of both classification accuracy and execution time, the method has three main limitations (that are also mentioned in the Introduction section) resulting in a low classification accuracy. Therefore, the existing methods have room for further improvement.

\section{Our Proposed Incremental Learning Framework: Adaptive Decision Forest (ADF)}
\label{Our_Framework}

In this section, we present our proposed incremental learning framework called  \textbf{A}daptive \textbf{D}ecision \textbf{F}orest (\textbf{ADF}), which produces a decision forest to classify new data. ADF takes as input the data that arrive as batches over time. We consider a batch data set $D^{t^i}_B$ that arrives at time $t^i$. We also consider that the batch data $D^{t^i}_B$ has a set of class values $C^{t^i}=\{C^{t^i}_1, C^{t^i}_2, \ldots C^{t^i}_k\}$ where a class value $C^{t^i}_v$ is associated with a record $X_i\in D^{t^i}_B$. Let $T^i$ be a decision tree which is built on $D^{t^i}_B$. Thus, the class values $C^{t^i}\in D^{t^i}_B$ are considered to be known to $T^i$.     

At time $t^j$, ADF receives another batch data set $D^{t^j}_B$ which we assume has the same attributes $A=\{A_1, A_2, \ldots A_m\}$. However, $D^{t^j}_B$ can have a different set of class values $C^{t^j}=\{C^{t^j}_1, C^{t^j}_2, \ldots C^{t^j}_k\}$. The class values $C^{t^j}$ are considered to be new (or unknown) to the decision tree, $T^i$. The class values $C^{t^j}$ of a new batch $D^{t^j}_B$ can be categorized into five scenarios, namely single known class (SKC), multiple known class (MKC), single unknown class (SUC), multiple unknown class (MUC) and a mixture of known and unknown classes (MKUC). The scenarios are defined in Eq.~\ref{class_scenarios}.
\begin{align}
  {Scenario(C^{t^i},C^{t^j})}_{i\neq j} = \left\{
  \begin{array}{l l}
    \text{SKC} & \quad \text{if $|C^{t^j}|=1$ \& $C^{t^j}_k\in C^{t^i}$}\\
    \text{MKC} & \quad \text{if $|C^{t^j}|>1$ \& $C^{t^j}_k\in C^{t^i}$;$\forall {k}$}\\
		\text{SUC} & \quad \text{if $|C^{t^j}|=1$ \& $C^{t^j}_k\notin C^{t^i}$}\\
    \text{MKC} & \quad \text{if $|C^{t^j}|>1$ \& $C^{t^j}_k\notin C^{t^i}$;$\forall {k}$}\\
		\text{MKUC} & \quad \text{if $|C^{t^j}|>1$ \& $C^{t^j}_k\in C^{t^i}$;$\forall {k}$ \& $C^{t^j}_v\notin C^{t^i}$;$\forall {v}$}\\
  \end{array} \right.  
	;\forall {i,j}
\label{class_scenarios}
\end{align}

The categorization of class values of a new batch data is shown in Fig.~\ref{fig:framework}. To the best of our knowledge, traditional class incremental learning methods including CIRF~\cite{hu2018novel} usually focus on the SUC scenario and do not consider other scenarios. ADF is capable of handling all five scenarios by integrating our novel $iSAT$ splitting strategy with an existing decision forest algorithm such as RF~\cite{breiman2001random}. 

\begin{figure}[ht!]
\centering
  \setlength{\belowcaptionskip}{0pt}
	\setlength{\abovecaptionskip}{0pt}	
	\includegraphics[width=0.90\linewidth]{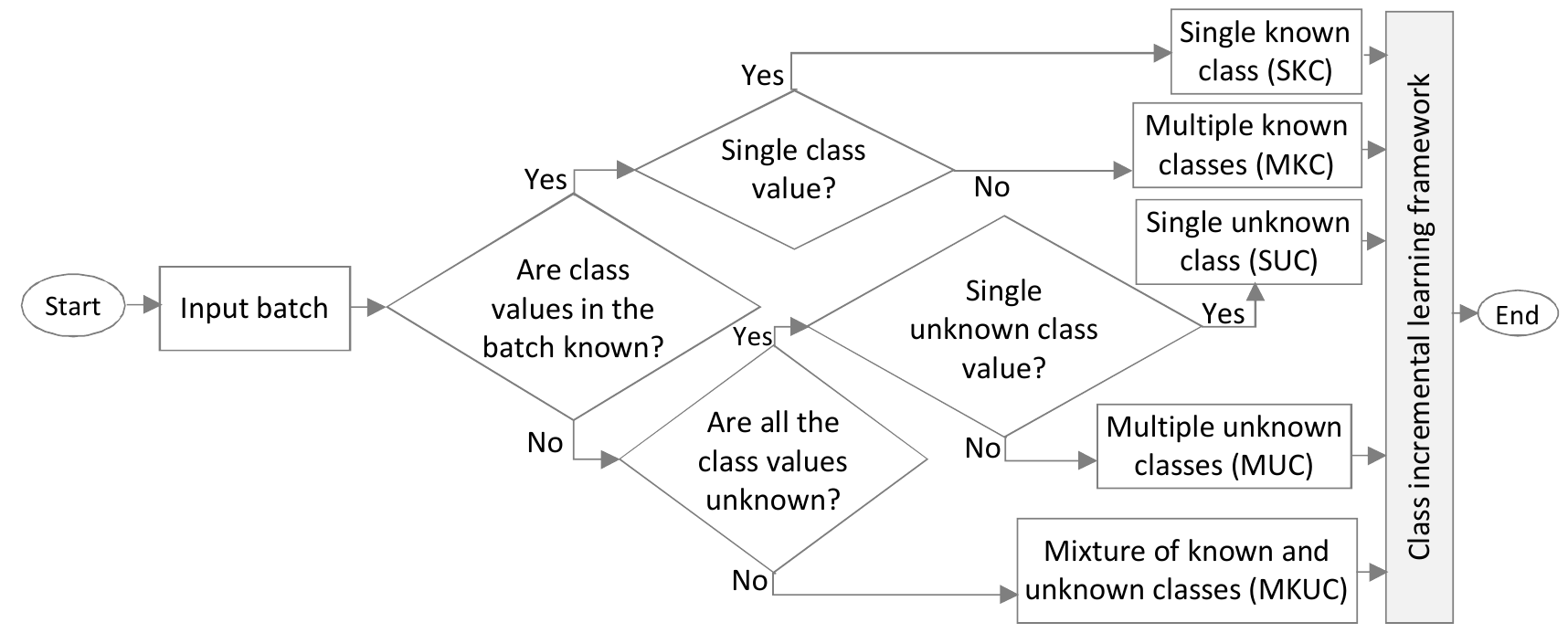}
	 \caption{Categorization of class values of a new batch data.}
	 \label{fig:framework}
\end{figure}

\subsection{Basic Concept of ADF}
\label{basicconcept}

 Incremental decision forest algorithms (such as CIRF~\cite{hu2018novel} and ARF~\cite{gomes2017adaptive}) can adapt the decision trees continuously and perform better than the traditional (non-incremental) decision forest algorithms (such as RF~\cite{breiman2001random} and SysFor~\cite{islam2011knowledge}) in terms of classification accuracy and training time~\cite{hu2018novel}. We observe the following three issues that play a vital role in achieving high accuracy by an incremental decision forest algorithm which is also susceptible to handling all five scenarios as shown in Eq.~\ref{class_scenarios} and Fig.~\ref{fig:framework}.

The first issue is to identify the best split for a new batch data $D^{t^i}_B$. The commonly used splitting strategies are entropy~\cite{quinlan1986induction}, Gini index~\cite{breiman1984classification} and axis aligned~\cite{breiman2001random}. An existing algorithm called CIRF~\cite{hu2018novel} makes use of an existing splitting strategy called Separating Axis Theorem (SAT)~\cite{gottschalk1996separating} to find the best split. However, we argue that CIRF does not perform well on $D^{t^i}_B$ if it contains records with a mix of known and unknown classes (MKUC). To find the best split for a batch that follows the MKUC scenario, we propose a novel splitting strategy called improved separating axis theorem (iSAT) (see section~\ref{isat}) by which it is expected to achieve a high classification accuracy.

We illustrate the argument by considering three toy batch data sets as shown in Fig.~\ref{fig:justify_sat_entropy_isat_tree}a, Fig.~\ref{fig:justify_sat_entropy_isat_tree}b and Fig.~\ref{fig:justify_sat_entropy_isat_tree}c. Each batch data set contains 10 records and three attributes, namely ``Area of a House in square meter (Area)'', ``number of bedrooms (Beds)'' and ``House rent category (Rent)''. Figure~\ref{fig:justify_sat_entropy_isat_tree}d shows a decision tree which is built from the batch data $B0$ (as shown in Fig.~\ref{fig:justify_sat_entropy_isat_tree}a). The decision tree is repaired based on the SAT splitting strategy for the batch data $B1$ (see Fig.~\ref{fig:justify_sat_entropy_isat_tree}b) which follows the MKUC scenario. The repaired decision tree is shown in Fig.~\ref{fig:justify_sat_entropy_isat_tree}e. Figure~\ref{fig:justify_sat_entropy_isat} shows that the classification accuracy of the repaired decision tree (as shown in Fig.~\ref{fig:justify_sat_entropy_isat_tree}b) is lower than the classification accuracy of the decision trees as shown in Fig.~\ref{fig:justify_sat_entropy_isat_tree}h and Fig.~\ref{fig:justify_sat_entropy_isat_tree}k that are repaired based on the splitting strategies, namely entropy and iSAT (see section~\ref{isat}), respectfully. Besides, the decision tree with the iSAT splitting strategy achieves high classification accuracy for all batches.

\begin{figure}[ht!]
\centering
  \setlength{\belowcaptionskip}{0pt}
	\setlength{\abovecaptionskip}{0pt}	
	\includegraphics[width=0.90\linewidth]{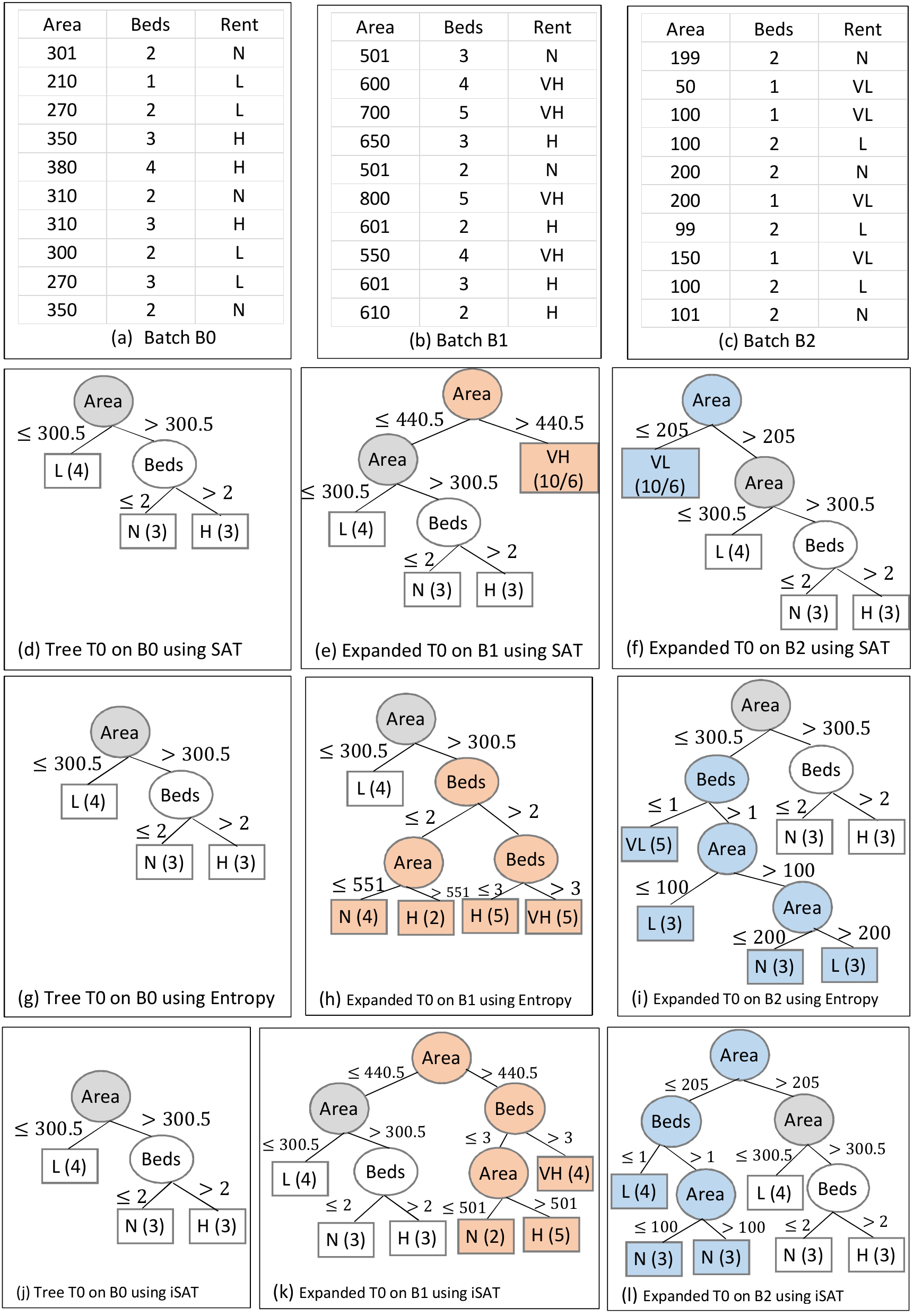}
	 \caption{Construction of trees on toy batches in terms of SAT, Entropy and iSAT splitting strategies.}
	 \label{fig:justify_sat_entropy_isat_tree}
\end{figure}

\begin{figure}[ht!]
\centering
  \setlength{\belowcaptionskip}{0pt}
	\setlength{\abovecaptionskip}{0pt}	
	\includegraphics[width=0.90\linewidth]{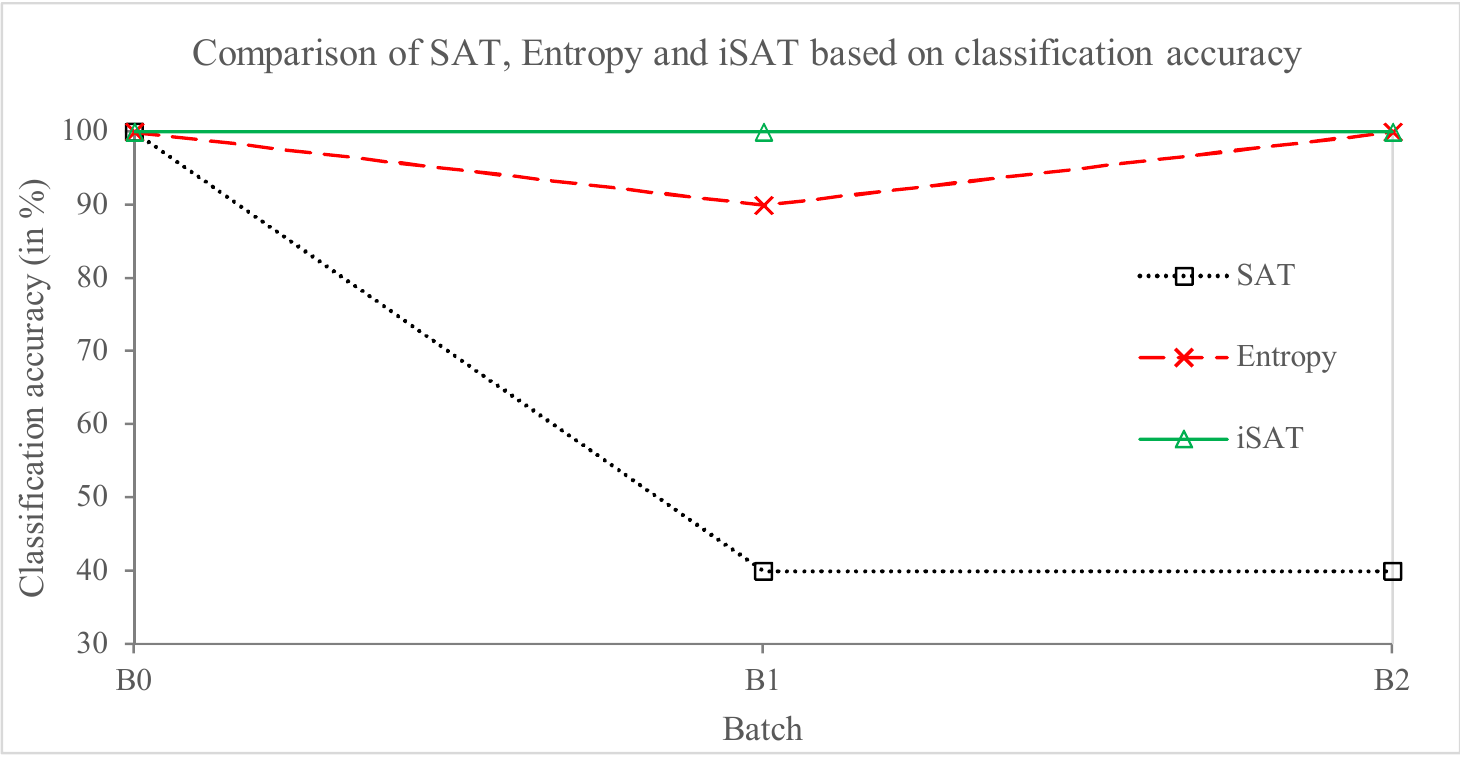}
	 \caption{Comparison of classification accuracies on toy batches in terms of SAT, Entropy and iSAT splitting strategies.}
	 \label{fig:justify_sat_entropy_isat}
\end{figure}

The second issue is to decide whether the decision trees are required to repair or not. Moreover, if the trees are required to repair then an important task is to decide how many of the trees should be repaired by considering the cost of modification. To address the issue, we introduce a repairable strategy in section~\ref{dtrs} to find an optimal repairable threshold for achieving a high accuracy with a low cost.   

The third issue is to handle concept drifts and preserving historical information that are gained from the previous batches. In dynamic data structure, data distribution may change over time and this situation is known as concept drift, which may generally be categorized in two types, namely Temporary Concept Drift (TCD)\footnote{TCD: If the concept drift occurs in only one or a few consecutive batches.} and Sustainable Concept Drift (SCD)\footnote{SCD: If the concept drift sustains over a period of time.}. To handle TCD, if a decision forest is rebuilt by considering only the current batch $D^{t^i}_B$ then it is most likely that the historical information of the forest will be lost. We argue that the decision forest is unable to classify the records if they are drawn from the distribution of the previous batches and may lead to a low classification accuracy. We also argue that the classification accuracy may be increased if the decision forest is rebuilt only for the SCD instead of the TCD. A strategy for identifying a SCD is presented in Section~\ref{identify_concept_drift}. Moreover, to achieve a high accuracy, a number of consecutive batches called window (discussed in Section~\ref{justify_window}) instead of just the current batch can be used while rebuilding the decision forest. 

We also argue that the historical information may be preserved by building a set of forests called Permanent Forest ($PF$), Active Forest ($AF$), and Temporary Forest ($TF$) which are discussed in Section~\ref{justify_pf_af_tf}. ADF can update $PF$, $AF$ and $TF$ in parallel and achieves a high accuracy by keeping the historical information.      

To address the issues, in ADF we consider a number of strategies that are summarized as follows: improved separating axis theorem (iSAT) based splitting strategy (in Section~\ref{isat}), decision trees repairable strategy (in Section~\ref{dtrs}), the use of a set of decision forests (in Section~\ref{justify_pf_af_tf}), identification of SCD strategy (in Section~\ref{identify_concept_drift}), and the use of a window of batches strategy (in Section~\ref{justify_window}). We now discuss each of these strategies as follows.

\subsubsection{Improved Separating Axis Theorem (iSAT) Splitting Strategy}
\label{isat}

We propose a novel splitting strategy called \textbf{iSAT} to find a splitting attribute and value for a node of a decision tree. Our proposed splitting strategy $iSAT$ handles the problems of an existing method CIRF~\cite{hu2018novel} which makes use of an existing algorithm called Axis Aligned Minimum Bounding Box (AABB)~\cite{arvo1990transforming} and an existing theorem called Separating Axis Theorem (SAT)~\cite{gottschalk1996separating} to find the best split attribute and value. We first briefly introduce AABB and SAT before $iSAT$.
\begin{definition}
\label{defaabb}
In geometry, the smallest surrounding box of a set of data points is known as the minimum bounding box~\cite{arvo1990transforming} where all the data points are surrounded by the box.     
\end{definition}

A sample AABB of a set of data points is illustrated in Fig.~\ref{fig:aabb}a.
 If two minimum bounding boxes are not intersected then it can be concluded that no overlap exists between the corresponding sets of data points. For example, Fig~\ref{fig:aabb}b shows two sets of data points: black stars and red triangles. The smallest surrounding boxes for the black stars and red triangles are marked with ${AABB}_1$ and ${AABB}_2$ where the boxes are overlapped at axis $A_j$ but not at $A_i$. Thus, the black stars and red triangles are separated by ${AABB}_1$ and ${AABB}_2$ at axis $A_i$.

The strategy of separating two sets of data points can be used for the incremental expansion of a decision tree by inserting a parent of a node. For example, Fig.~\ref{fig:justify_sat_entropy_isat_tree} demonstrates that the decision tree shown in Fig.~\ref{fig:justify_sat_entropy_isat_tree}d is expanded by adding a parent node ``Area'' at the root node and the resulting tree is shown in Fig.~\ref{fig:justify_sat_entropy_isat_tree}e. Using AABBs, a decision tree can be represented as a hierarchical form of nested bounding boxes~\cite{hu2018novel}. For $D^{t^i}_B$ with new class values, it is crucial to identify the best split attribute and value. The issue can be addressed by using the SAT~\cite{gottschalk1996separating} as follows.

\begin{theorem}

Separating Axis Theorem (SAT)~\cite{gottschalk1996separating, hu2018novel, boyd2004convex}: If two nonempty convex objects are disjoint then there exists an axis on which the projection of the convex objects will not overlap.
\label{sat_theorem}
\end{theorem}

\begin{proof}

Let $R$ be a set of real numbers and $R^n$ be a set of real n-vectors. Also let $z\in R^n$ be a normal vector and $b\in R$ be a real number. We now consider two nonempty convex objects Q and S where $Q\cap S=\emptyset$. Then there exist $z\neq 0$ and $b$ such that $x|z^Tx\leq b;\forall x\in Q$ and $x|z^Tx\geq b;\forall x\in S$. That is the affine function, $f(x)=z^Tx-b$ is nonpositive on Q and nonnegative on S. The hyperplane, $h$, $\{x|z^Tx=b\}$ is considered as the separating hyperplane for the two objects Q and S. In other words, the hyperplane, $h$, separates the convex objects Q and S which is illustrated in Fig.~\ref{fig:proof_sat_a}.

\begin{figure}[ht!]
\centering
  \setlength{\belowcaptionskip}{0pt}
	\setlength{\abovecaptionskip}{0pt}	
	    \subfigure[]
	    {
	        \includegraphics[width=0.20\linewidth]{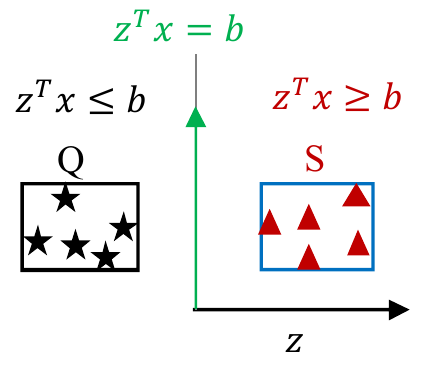}
	        \label{fig:proof_sat_a}
	    }
	    \subfigure[]
	    {
	        \includegraphics[width=0.20\linewidth]{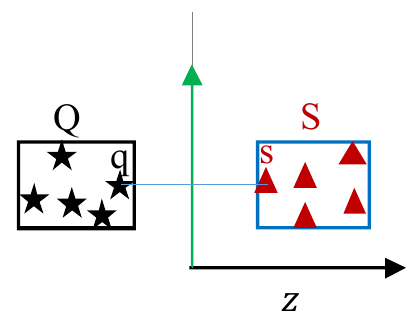}
	        \label{fig:proof_sat_b}
	    }
			\subfigure[]
	    {
	        \includegraphics[width=0.20\linewidth]{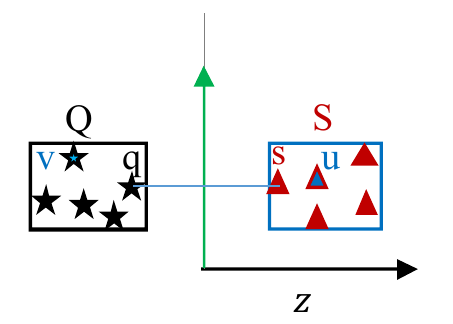}
	        \label{fig:proof_sat_c}
	    }
			
	    \caption{Construction of a separating hyperplane, $h$, $\{x|z^Tx=b\}$ between two disjoint convex objects Q and S~\cite{boyd2004convex}. (a) The hyperplane, $h$, separates the convex objects Q and S. The affine function $z^Tx-b$ is nonpositive on Q and nonnegative on S. (b) The points $q\in Q$ and $s\in S$ are the pair points in the two sets that are closest to each other, where the line segment $\overline{qs}$ between q and s is bisected by the separating hyperplane $h$ and $h\perp z$. (c) For a point $u\in S$, the affine function $f(u)$ is nonnegative on S and for a point $v\in Q$, the affine function $f(v)$ is nonpositive on Q.}
	    \label{fig:proof_sat}
\end{figure}

Let $u\in S$ and $v\in Q$ be two points. The Euclidean distance between the points is calculated as
 \begin{align}
  dist(u, v) = {\left\|u-v\right\|}_2 
	\label{sat_distance_eq}	
\end{align}
We consider a vector $DIST$ that contains the distances between the points of Q and S. There exists points $q\in Q$ and $s\in S$ which has the minimum distance between Q and S that is 
\begin{align}
  mindist(Q, S) = \inf\{DIST\} 
	\label{sat_distance_eq0}	
\end{align}
We now define
\begin{center}
$z=s-q$,~~~~~ $b=\frac{\left\|s\right\|_2^2 - \left\|q\right\|_2^2}{2}$ 
\end{center}
Thus by substituting the value of $z$ and $b$, the affine function $f(x)=z^Tx-b$ can be expressed as  
\begin{align}
  f(x)=(s-q)^T(x-\frac{s+q}{2})
	\label{sat_affine_eq1}	
\end{align}
The goal is to show that $f(x)$ is nonpositive on Q and nonnegative on S that is the hyperplane, $h$, $\{x|z^Tx=b\}$ is perpendicular to the normal vector $z$ and bisects the line segment $\overline{qs}$ between q and s as shown in Fig.~\ref{fig:proof_sat_b}.

We first prove that $f(x)$ is nonnegative on S. For any point $u\in S$ (as shown in Fig.~\ref{fig:proof_sat_c}), the affine function $f(u)$ can be written as 
 \begin{align}
  f(u) = (s-q)^T(u-\frac{s+q}{2}) > 0
\label{sat_affine_eq2}	
\end{align}

The affine function $f(u)$ can also be expressed as
  \begin{align}
  f(u) = (s-q)^T(u-s+\frac{s-q}{2}) =(s-q)^T(u-s)+\frac{\left\|s-q\right\|_2^2}{2}
\label{sat_affine_eq3}	
\end{align}

Now the Euclidean distance between $u\in S$ and $q\in Q$ can be calculated by taking the derivative of $f(u)$ for a closed interval $t$ (with $t$ as close to zero) as follows
\begin{align}
  \frac{d}{dt}\left\|s+t(u-s)-q\right\|_2^2|_{t=0}=2(s-q)^T(u-s)>0
\label{sat_affine_eq4}	
\end{align}
So for some small $t~(0<t\leq 1)$, we have
\begin{align}
  \left\|s+t(u-s)-q\right\|_2 > \left\|s-q\right\|_2
\label{sat_affine_eq5}	
\end{align}
that is the point $s+t(u-s)$ is closer to s than q. Since S is convex and contains s and u, we have $s+t(u-s)\in S$. Thus, the affine function $f(x) \forall x\in S$ is nonnegative on S. Similarly, for any point $x\in Q$, the affine function $f(x) \forall x\in Q$ is nonpositive on Q. Hence, the hyperplane, $h$ separates the convex objects S and Q. Therefore, there exists an axis $z\neq 0$ on which the projection of two disjoint convex objects S and Q does not overlap. 
\label{sat_proof}
\end{proof}

To illustrate Theorem~\ref{sat_theorem}, we consider two dimensional AABBs as shown in Fig~\ref{fig:aabb}b where two attributes $A_i$ and $A_j$ are projected on the x-axis and y-axis, respectively. ${AABB}_1$ is created based on black stars and ${AABB}_2$ is created based on red triangles. From the figure we can see that there is no overlap between ${AABB}_1$ and ${AABB}_2$ for the $A_j$ attribute. According to SAT, the x-axis is the separating axis between ${AABB}_1$ and ${AABB}_2$ and thus $A_i$ is the splitting attribute. Moreover, if multiple attributes have no overlap then the attribute having the maximum gap will be considered as the splitting attribute. Let $d_1$ be the gap between ${AABB}_1$ and ${AABB}_2$ on the x-axis (i.e. attribute $A_i$ ) and $d_2$ be the gap between ${AABB}_1$ and ${AABB}_2$ on the y-axis (i.e. attribute $A_j$ ) then the split attribute $A_s$ is determined based on Eq.~\eqref{sat_splitAttr_eq}~\cite{hu2018novel}.

\begin{align}
  A_{s} = \left\{
  \begin{array}{l l}
    A_i & \quad \text{if $d_1>d_2$ \& $d_1>0$}\\
    A_j & \quad \text{if $d_2>d_1$ \& $d_2>0$}\\
  \end{array} \right.
  ;\forall {i,j}
\label{sat_splitAttr_eq}
\end{align}

For the x-axis (attribute $A_i$), let $V^{i,1}_{min}$ and $V^{i,1}_{max}$ be the minimum and maximum values of ${AABB}_1$, and $V^{i,2}_{min}$ and $V^{i,2}_{max}$ be the minimum and maximum values of ${AABB}_2$. Similarly, for the y-axis (attribute $A_j$), let $V^{j,1}_{min}$ and $V^{j,1}_{max}$ be the minimum and maximum values of ${AABB}_1$, and $V^{j,2}_{min}$ and $V^{j,2}_{max}$ be the minimum and maximum values of ${AABB}_2$. Thus, the split value $SV$ for $A_s$ is calculated by using Eq.~\eqref{sat_splitVal_eq}~\cite{hu2018novel}.

\begin{align}
  SV = \left\{
  \begin{array}{l l}
    \frac{V^{i,1}_{max}+V^{i,2}_{min}}{2} & \quad \text{if $d_1>d_2$ \& $V^{i,1}_{max}>V^{i,2}_{min}$}\\
		\frac{V^{i,2}_{max}+V^{i,1}_{min}}{2} & \quad \text{if $d_1>d_2$ \& $V^{i,2}_{min}>V^{i,1}_{max}$}\\
    \frac{V^{j,1}_{max}+V^{j,2}_{min}}{2} & \quad \text{if $d_2>d_1$ \& $V^{j,1}_{max}>V^{j,2}_{min}$}\\
		\frac{V^{j,2}_{max}+V^{j,1}_{min}}{2} & \quad \text{if $d_2>d_1$ \& $V^{j,2}_{min}>V^{j,1}_{max}$}\\
  \end{array} \right.
  ;\forall {i,j}
\label{sat_splitVal_eq}
\end{align}

Using the SAT based splitting strategy, incremental decision forest algorithms (such as CIRF~\cite{hu2018novel}) find the split attribute which has the maximum distance among all numerical attributes. For example, the two bounding boxes ${AABB}_1$ and ${AABB}_2$ are separated by the attribute $A_i$ in Fig~\ref{fig:aabb}a. However, the algorithms do not perform well on $D^{t^i}_B$ with MKUC scenario (as discussed in Section~\ref{basicconcept}). Moreover, the algorithms fail to find a split if the bounding boxes are non-separable. For example, Fig.~\ref{fig:aabb}d shows that two bounding boxes ${AABB}_1$ and ${AABB}_2$ that are not separated by any attributes. 

We argue that the accuracy of such algorithms can be improved if it is possible to find the best split in the case of overlapping bounding boxes. We propose an improved separating axis theorem (iSAT) where we propose a definition and two theorems. In iSAT we first define a sub-box called Sub Axis aligned minimum bounding box (Sub-AABB) as follows.  

\begin{definition}
\label{defsubbox}
If we have a set of data points then a Sub-AABB is the smallest axis aligned minimum bounding box of a subset of the set of data points.      
\end{definition}

Three sample Sub-AABBs of an AABB are illustrated in Fig.~\ref{fig:sub_box_nonoverlap_proof}a where the Sub-AABBs are marked with the dotted red rectangles. According to the Definition~\ref{defsubbox}, two overlapping AABBs (as shown in Fig.~\ref{fig:aabb}d) can have a number of overlapping and non-overlapping Sub-AABBs as illustrated in Fig.~\ref{fig:sub_box_nonoverlap_proof}b.

\begin{figure}[ht!]
\centering
  \setlength{\belowcaptionskip}{0pt}
	\setlength{\abovecaptionskip}{0pt}	
	 \subfigure[]
	    {
	        \includegraphics[width=0.15\linewidth]{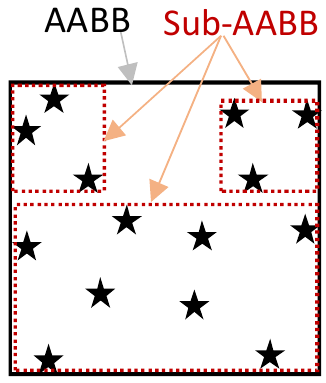}
	        \label{fig:sub_box}
	    }
	    \subfigure[]
	    {
	        \includegraphics[width=0.30\linewidth]{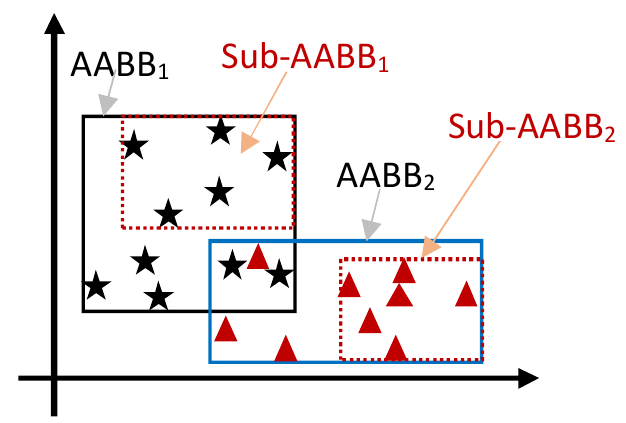}
	        \label{fig:sub_box_non_overlap}
	    }
			\subfigure[]
	    {
	        \includegraphics[width=0.30\linewidth]{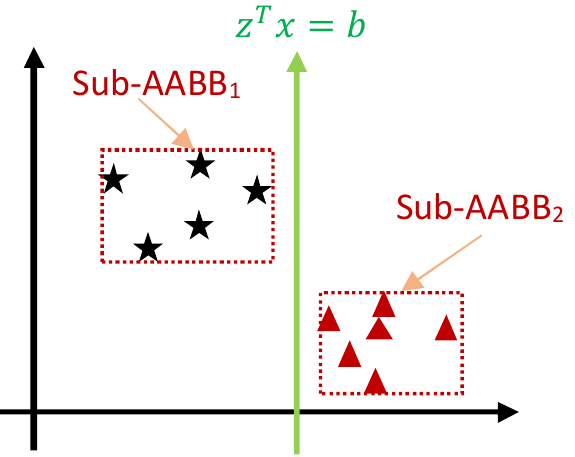}
	        \label{fig:sub_box_non_overlap_proof}
	    }
	 \caption{(a) A 2-dimensional illustration of Sub-AABBs. (b) The ${AABB}_1$ and ${AABB}_2$ consist of sub-AABBs ${Sub-AABB}_1$ and ${Sub-AABB}_2$, respectively. (c) The hyperplane separates the two sub-AABBs ${Sub-AABB}_1$ and ${Sub-AABB}_2$.}
	 \label{fig:sub_box_nonoverlap_proof}
\end{figure}

\begin{theorem}
If there are two overlapping AABBs then it is possible to have non-overlapping Sub-AABBs of the two AABBs.
\label{subbox_nonoverlap_theorem}
\end{theorem}
\begin{proof}
We consider two nonempty convex objects ${AABB}_1$ and ${AABB}_2$ where ${AABB}_1\cap {AABB}_2 \neq \emptyset$ as illustrated in Fig.~\ref{fig:aabb}d. According to Definition~\ref{defsubbox}, both ${AABB}_1$ and ${AABB}_2$ may have a number of Sub-AABBs. Let ${Sub-AABB}_1$ and ${Sub-AABB}_2$ be two Sub-AABBs where ${Sub-AABB}_1\cap {Sub-AABB}_2 = \emptyset$ as illustrated in Fig.~\ref{fig:sub_box_nonoverlap_proof}b. Then there exist $z\neq 0$ and $b$ such that $x|z^Tx\leq b;\forall x\in {Sub-AABB}_1$ and $x|z^Tx\geq b;\forall x\in {Sub-AABB}_2$. That is the affine function, $f(x)=z^Tx-b$ is nonpositive on ${Sub-AABB}_1$ and nonnegative on ${Sub-AABB}_2$. The hyperplane, $h$, $\{x|z^Tx=b\}$ is considered as the separating hyperplane for the two objects ${Sub-AABB}_1$ and ${Sub-AABB}_2$. According to Theorem~\ref{sat_theorem}, for any point $x\in {Sub-AABB}_2$ the affine function $f(x) \forall x\in {Sub-AABB}_2$ is nonnegative on ${Sub-AABB}_2$. Similarly, for any point $x\in {Sub-AABB}_1$, the affine function $f(x) \forall x\in {Sub-AABB}_1$ is nonpositive on ${Sub-AABB}_2$. Hence, the hyperplane, $h$, separates the sub-AABBs ${Sub-AABB}_1$ and ${Sub-AABB}_2$. Therefore, there exists an axis $z\neq 0$ on which the projection of two disjoint sub-AABBs ${Sub-AABB}_1$ and ${Sub-AABB}_2$ does not overlap which is illustrated in Fig.~\ref{fig:sub_box_nonoverlap_proof}c.
\label{subbox_nonoverlap_theorem_proof}
\end{proof}

\begin{theorem}
If there are two overlapping AABBs, then there will be at least two overlapping Sub-AABBs of the two AABBs.
\label{subbox_overlap_theorem}
\end{theorem}
\begin{proof}
We again consider two nonempty convex objects ${AABB}_1$ and ${AABB}_2$ where ${AABB}_1\cap {AABB}_2 \neq \emptyset$ as illustrated in Fig.~\ref{fig:aabb}d. According to Definition~\ref{defsubbox}, both ${AABB}_1$ and ${AABB}_2$ may have a number of Sub-AABBs. Let ${Sub-AABB}_3$ and ${Sub-AABB}_4$ be two Sub-AABBs where ${Sub-AABB}_3\cap {Sub-AABB}_4 \neq \emptyset$ as illustrated in Fig.~\ref{fig:sub_aabb_overlap_proof}a. For simplicity, we consider a two dimensional projection of ${Sub-AABB}_3$ and ${Sub-AABB}_4$ on the x-axis and y-axis as shown in Fig.~\ref{fig:sub_aabb_overlap_proof}b. On the x-axis, let $x_1$ and $x_3$ be the lower bound and upper bound, respectively of ${Sub-AABB}_3$ and $x_2$ and $x_4$ be the lower bound and upper bound, respectively of ${Sub-AABB}_4$ such that $x_1<x_2<x_3<x_4$. On the y-axis, let $y_1$ and $y_3$ be the lower bound and upper bound, respectively of ${Sub-AABB}_4$ and $y_2$ and $y_4$ be the lower bound and upper bound, respectively of ${Sub-AABB}_3$ such that $y_1<y_2<y_3<y_4$.
Let $u$ and $v$ be two points of ${Sub-AABB}_3$ and ${Sub-AABB}_4$, respectively. If $\left\|v\right\|_2^2<\left\|u\right\|_2^2$ for any points $u\in {Sub-AABB}_3$ and $v\in {Sub-AABB}_4$ then there does not exist any axis for which the two Sub-AABBs are separable. 

On the x-axis, let $u=x_3$ and $v=x_2$. Since $x_2<x_3$ then $\left\|v\right\|_2^2<\left\|u\right\|_2^2$. Thus, ${Sub-AABB}_3$ and ${Sub-AABB}_4$ are not separable on the x-axis.

Similarly, on the y-axis, let $u=y_3$ and $v=y_2$. Since $y_2<y_3$ then $\left\|v\right\|_2^2<\left\|u\right\|_2^2$. Thus, ${Sub-AABB}_3$ and ${Sub-AABB}_4$ are not separable on the y-axis. Therefore, it can be concluded that the two overlapping AABBs consist at least two overlapping Sub-AABBs.  

\label{subbox_overlap_theorem_proof}
\end{proof}
\begin{figure}[ht!]
\centering
  \setlength{\belowcaptionskip}{0pt}
	\setlength{\abovecaptionskip}{0pt}	
	 \subfigure[]
	    {
	        \includegraphics[width=0.30\linewidth]{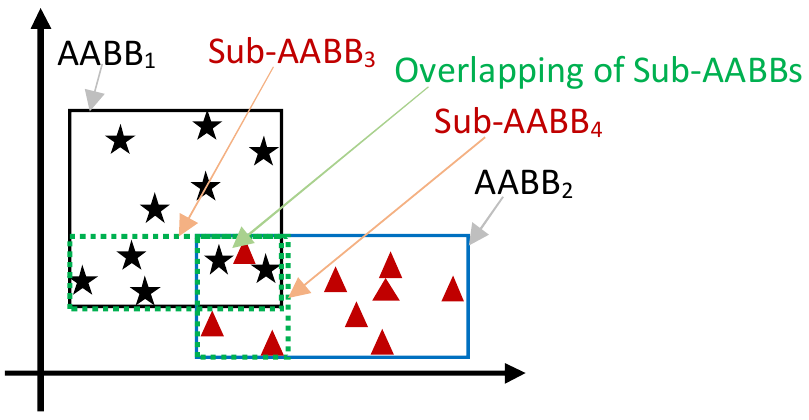}
	        \label{fig:sub_box_overlap1}
	    }
	    \subfigure[]
	    {
	        \includegraphics[width=0.30\linewidth]{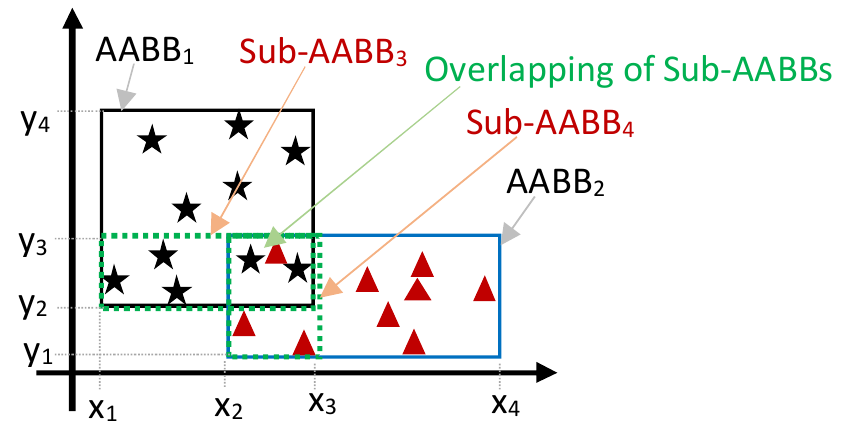}
	        \label{fig:sub_box_overlap_proof}
	    }
	    \subfigure[]
	    {
	        \includegraphics[width=0.30\linewidth]{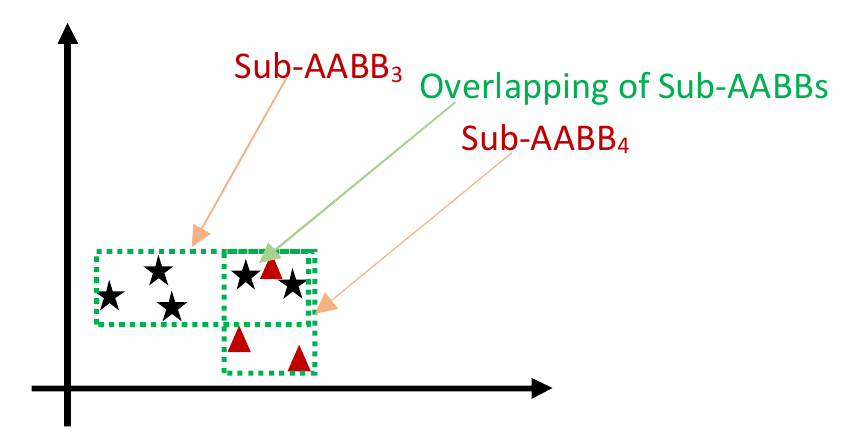}
	        \label{fig:sub_box_overlap2}
	    }
		
	 \caption{(a)  The ${AABB}_1$ and ${AABB}_2$ consist of overlapping Sub-AABBs. (b) The boundaries of AABBs and Sub-AABBs are projected on the x-axis and y-axis. (c) ${Sub-AABB}_1$ and ${Sub-AABB}_2$ are overlapping. }
	 \label{fig:sub_aabb_overlap_proof}
\end{figure}

Following Theorem~\ref{subbox_nonoverlap_theorem} and Theorem~\ref{subbox_overlap_theorem} we realise that the best split-attribute ($A_s$) and split-value ($SV$) can be determined by considering the overlapping and non-overlapping Sub-AABBs. For non-overlapping Sub-AABBs, we use the SAT based splitting strategy to find the best $A_s$ and $SV$. On the other hand, for overlapping Sub-AABBs, we first identify the Sub-AABB which has the minimum number of records that have a mix of class values. For example, in Fig~\ref{fig:sub_aabb_overlap_proof}c, we can see that ${Sub-AABB}_4$ has the minimum number of records that have a mix of class values. For ${Sub-AABB}_4$ we use the entropy based splitting strategy to find the best $A_s$ and $SV$. The steps of the iSAT is shown in Algorithm~\ref{algo_iSAT}. The experimental results on the toy data set (see Fig.\ref{fig:justify_sat_entropy_isat}) also justifies the use of both SAT and entropy based splitting strategies while modifying the classifier $T$ incrementally.

\subsubsection{Decision Trees Repairable Strategy}
\label{dtrs}

We propose a novel strategy in our framework ADF to decide whether a decision forest, $T$ is repairable or not. We first calculate the confidence of all leaves of the decision forest, $T$. Let $L^{c^{t^i}}_{pq}$ be the confidence of the $q$-th leaf of the $p$-th tree for the batch data set $D^{t^i}_B$ at time $t^i$. For a new batch data set $D^{t^j}_B$ at time $t^j$, we also calculate the confidence of all leaves of the decision forest, $T$. Let $L^{c^{t^j}}_{pq}$ be the confidence of the $q$-th leaf of the $p$-th tree for $D^{t^j}_B$. We then calculate a repairable matrix $F$ for $T$ by comparing confidences that are calculated from the batch data sets $D^{t^i}_B$ and $D^{t^j}_B$. A zero value in an element $F_{pq}\in F$ indicates that the $q$-th leaf of the $p$-th tree is not perturbed, whereas a non-zero value indicates that the $q$-th leaf of the $p$-th tree is perturbed. The value of the elements $F_{pq}\in F$ is calculated by using Eq.~\ref{faultymatrix}.

\begin{align}
  F_{pq} = \left\{
  \begin{array}{l l}
    1 & \quad \text{if $L^{c^{t^i}}_{pq} > (L^{c^{t^j}}_{pq}+\epsilon)$}\\
    0 & \quad \text{Otherwise}\\
  \end{array} \right.
  ;\forall {p,q}
\label{faultymatrix}
\end{align}
where $\epsilon$ is an error tolerance threshold. 

We now calculate the ratio of perturbed leaves of the decision forest, $T$. Let $l$ be the total leaves of a tree and $l_{total}$ be the total leaves of the decision forest. Also let $F_{total} \leftarrow \sum_{\forall p,q} {F_{pq}}$ be the total number of perturbed leaves in $T$. We calculate the ratio of perturbed leaves $\vartheta\leftarrow F_{total}/l_{total}$ in $T$. If the value of $\vartheta$ is less than or equal to a user-defined threshold, $\theta$ (also called as the repairable threshold) then $T$ is considered as repairable. 

\begin{figure}[ht!]
\centering
  \setlength{\belowcaptionskip}{0pt}
	\setlength{\abovecaptionskip}{0pt}	
	    \subfigure[LDPA dataset]
	    {
	        \includegraphics[width=0.45\linewidth]{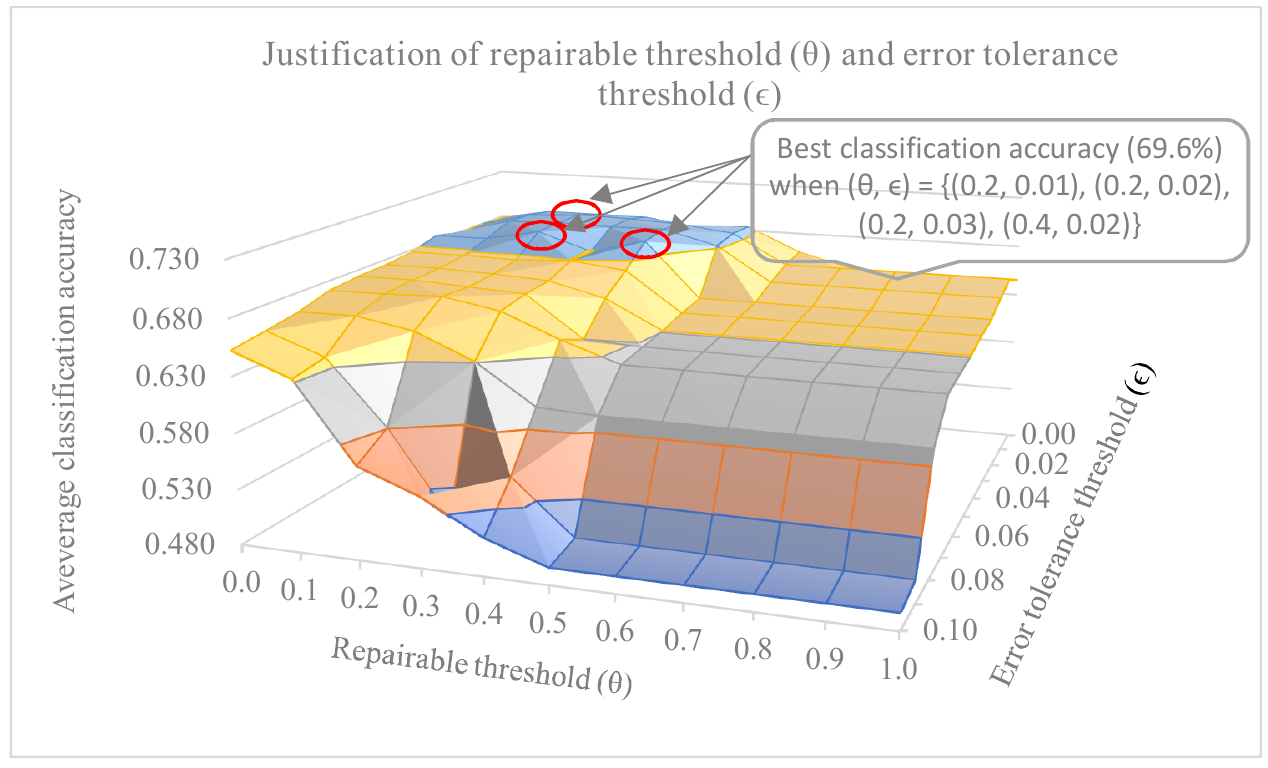}
	        \label{fig:Justify_rt_et_LDPA_accuarcy}
	    }
	    \subfigure[Avila dataset]
	    {
	        \includegraphics[width=0.45\linewidth]{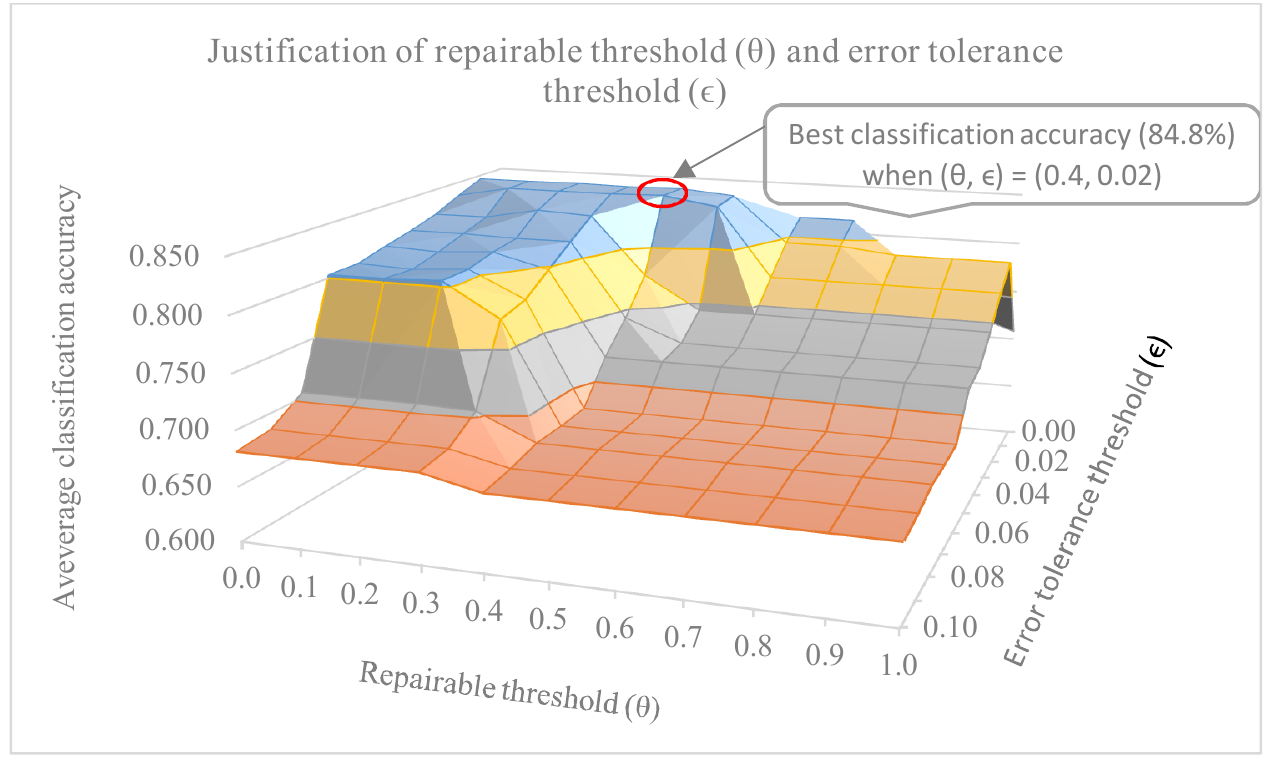}
	        \label{fig:Justify_rt_et_Avila_accuarcy}
	    }
	    \caption{Justification of repairable threshold $\theta$ and error tolerance threshold $\epsilon$ on LDPA and Avia data sets in terms of classification accuracy.}
	    \label{fig:justify_rt_et}
\end{figure}

We test the influences of different repairable threshold ($\theta$) and error tolerance threshold ($\epsilon$) values on two real data sets, namely LDAP~\cite{Frank+Asuncion:2010}, and Avila~\cite{Frank+Asuncion:2010} in terms of classification accuracy. From the data sets, we create 34 batches of training data and testing data (for details see Section~\ref{Experimental_settings}). We incrementally apply ADF on the training and testing batches and calculate the classification accuracies. We use the user-defined $\theta$ that varies between 0 and 1, and $\epsilon$ that varies between 0.0 and 0.1. For each combination of $\theta$ and $\epsilon$, we calculate average classification accuracy (from all 34 batches). The average classification accuracies on two data sets are presented in Fig.~\ref{fig:justify_rt_et}. Note that the perfect combination of $\theta$ and $\epsilon$ enables the ADF to perform the best. It is clear from the Fig.~\ref{fig:justify_rt_et} that we get the best classification accuracy (marked with circles) when $\theta$ is between 0.2 and 0.4 and $\epsilon$ is between 0.01 and 0.03 for both data sets.

\subsubsection{Handling concept drifts using three parallel forests}
\label{justify_pf_af_tf}

We explained in the basic concept (see Section~\ref{basicconcept}) that the distribution of data in dynamic situations may change over time, and Temporary Concept Drift (TCD) and Sustainable Concept Drift (SCD) may occur. In ADF, we handle both TCD and SCD by introducing three parallel decision forests, namely Permanent Forest (PF), $T^P$, Active Forest (AF), $T^A$,  and Temporary Forest (TF), $T^T$. The main motivation of the parallel forests is to keep historical knowledge leading to a high classification accuracy.  

The mechanism of ADF in parallel learning $T^P$, $T^A$ and $T^T$ is shown in Fig.~\ref{fig:justify_parallel_classifier}. For every new batch, $T^P$ is adapted. Besides, if $T^A$ is repairable, we adapt $T^A$ for the new batch; otherwise we check the adaptability of $T^T$, if $T^T$ is repairable, we adapt $T^T$ with the new batch; otherwise, we build a $T^T$ using a window of batches (see Section~\ref{justify_window}). Among the three forests, ADF suggests the user with the forest which achieves the best classification accuracy. 

\begin{figure}[ht!]
\centering
  \setlength{\belowcaptionskip}{0pt}
	\setlength{\abovecaptionskip}{0pt}	
	\includegraphics[width=0.90\linewidth]{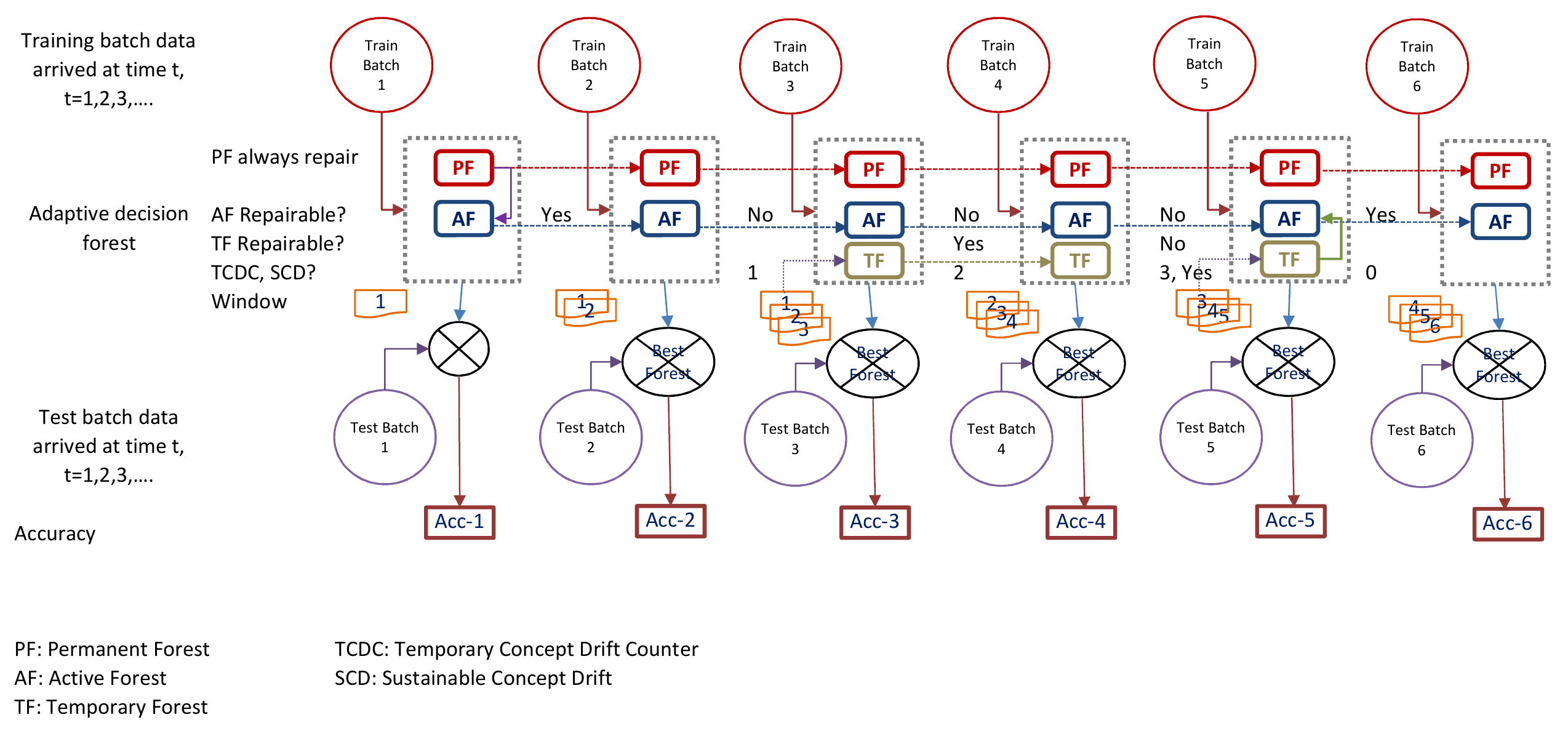}
	 \caption{Parallel incremental learning of $T^P$, $T^A$ and $T^T$.}
	 \label{fig:justify_parallel_classifier}
\end{figure}

We test the use of parallel forests on the LDPA data set. We create 34 batches of training and test data as described in Section~\ref{Experimental_batch_creation}. The performances of PF, PF, TF, and ADF on the LDPA data set is shown in Fig.~\ref{fig:justify_parallel_classifier_exp}. We can see that TF achieves the highest accuracy for the batches where AF is not repairable. Although AF is repaired for some batches (27 to 29), PF achieves the highest accuracy. The experimental result indicates the importance of the use of parallel forests in ADF.   

\begin{figure}[ht!]
\centering
  \setlength{\belowcaptionskip}{0pt}
	\setlength{\abovecaptionskip}{0pt}	
	\includegraphics[width=1.0\linewidth]{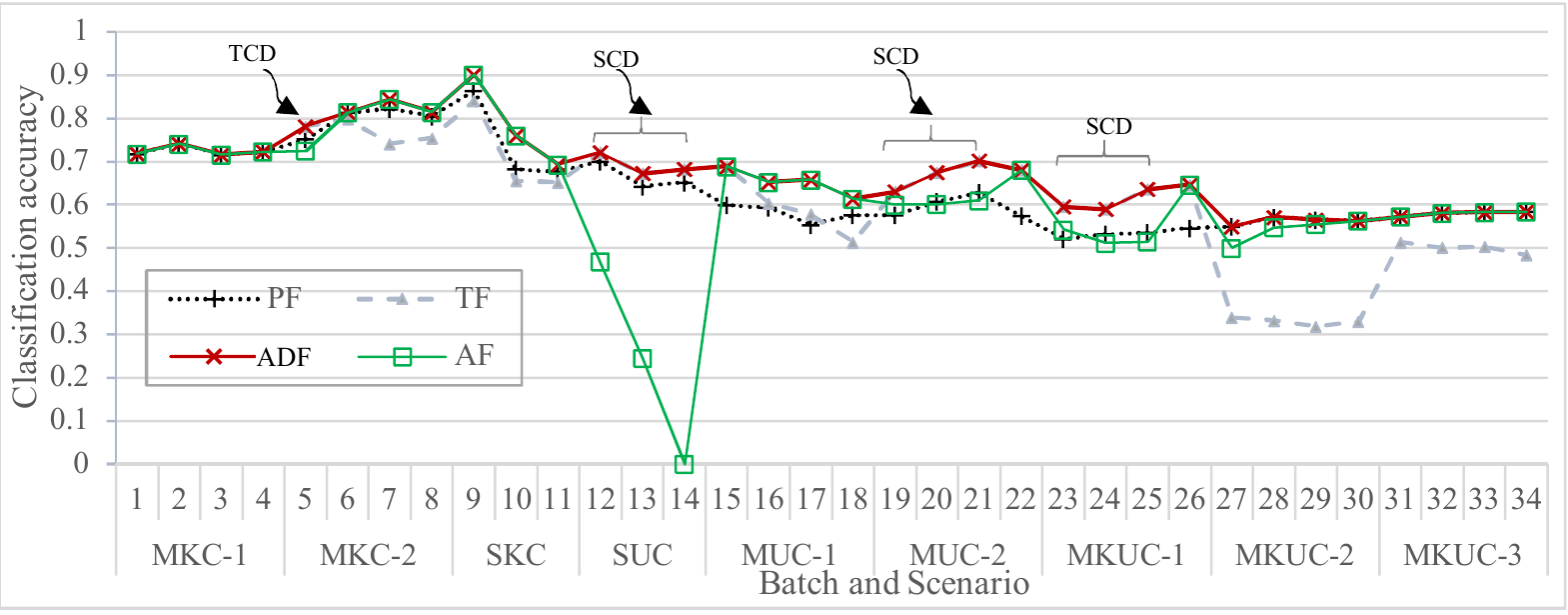}
	 \caption{Classification accuracies of PF, AF, TF and ADF on LDPA data set and the identification of temporary concept drift (TCD) and sustainable concept drift (SCD).}
	 \label{fig:justify_parallel_classifier_exp}
\end{figure}

\subsubsection{Identification of Sustainable Concept Drift (SCD)}
\label{identify_concept_drift}
The repairable strategy of ADF (see Section~\ref{dtrs}) enables it to determine the event of Sustainable Concept Drift (SCD) which is discussed in Section~\ref{basicconcept}. Let $cdf$ be the counter which counts the number of consecutive batches for which the active forest $AF$ (see Section~\ref{justify_pf_af_tf}) is not repairable. The event of SCD is determined by using Eq.~\ref{eq_concept_drift}.

\begin{align}
  SCD = \left\{
  \begin{array}{l l}
    true & \quad \text{if $cdf > \lambda$}\\
    false & \quad \text{Otherwise}\\
  \end{array} \right.
\label{eq_concept_drift}
\end{align}
where, $\lambda$ be the user-defined concept drift threshold. The default value of $\lambda$ is set to 3. Once a SCD is determined, we replace $T^A$ with $T^T$. However, if $cdf \leq \lambda$ and $T^A$ is repairable, we reset $cdf$ and $T^T$. Figure~\ref{fig:justify_parallel_classifier_exp} shows the detection of SCD on LDPA data set.

\subsubsection{Analysis of impact of using a window of batches}
\label{justify_window}

For handling SCD, we replace active forest, $T^A$, with temporary forest, $T^T$ which is rebuilt in case of TCD (see Section~\ref{justify_pf_af_tf}). We argue that if the data of a current batch are drawn from the distribution of the previous batches, the newly build $T^T$ may not classify the records correctly and may lead to a low classification accuracy. Therefore, to keep the accuracy high, we rebuild $T^T$ based on a window of batches. In ADF, we introduce a user-defined threshold $\gamma$ which is the size of the window. The default value of $\gamma$ is 3.

We test the influence of using a window of batches in ADF on a synthetic data set (see Table~\ref{tab:Datasets}) by creating 8 batches of training data and testing data (for details see Section~\ref{Experimental_settings}). For a batch $D^{t^i}_B$, we build a decision forest by applying SysFor~\cite{islam2011knowledge} on ${D^{t^i}_{B_{train}}}$ and noted the execution time. The accuracy of the forest is calculated by classifying the records of $D^{t^i}_{B_{test}}$. 
Similarly, we apply SysFor~\cite{islam2011knowledge} on a window that contains the last $\gamma$ number of batches and calculate the classification accuracy and execution time. Figure~\ref{fig:justify_window_rf} shows that SysFor (with a window of batches) achieves a higher classification accuracy at the expense of a slightly higher execution time. The result indicates the importance of using the window of batches while rebuilding $T^T$ in ADF.

\begin{figure}[ht!]
\centering
  \setlength{\belowcaptionskip}{0pt}
	\setlength{\abovecaptionskip}{0pt}	
	    \subfigure[Classification accuracies of SysFor~\cite{islam2011knowledge} and SysFor (Window) methods on Synthetic data set]
	    {
	        \includegraphics[width=0.45\linewidth]{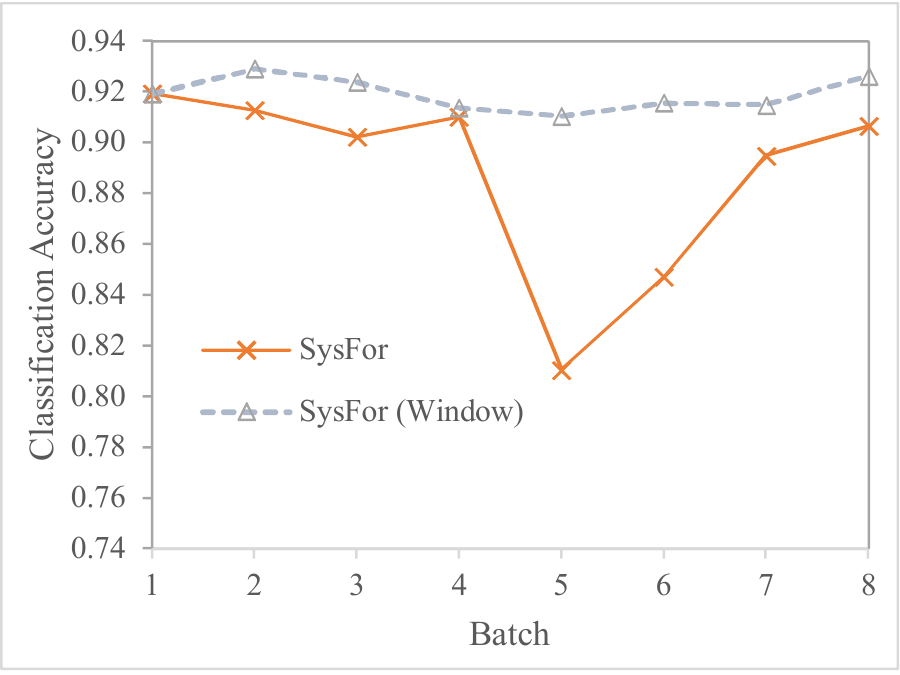}
	        \label{fig:justification_window_synthetic_accuracy}
	    }
	    \subfigure[Execution time of SysFor and SysFor (Window) methods on Synthetic data set]
	    {
	        \includegraphics[width=0.45\linewidth]{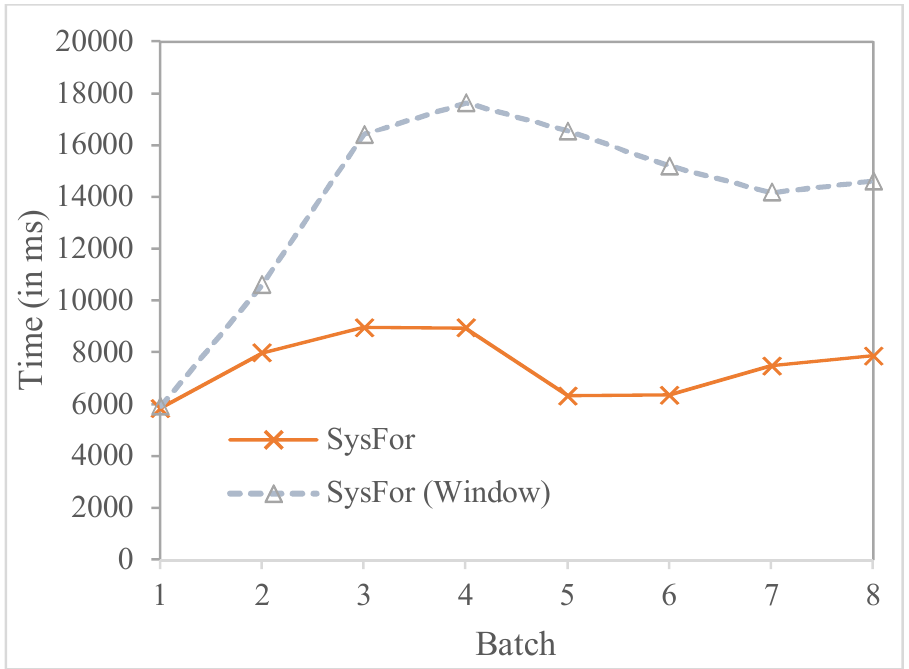}
	        \label{fig:justification_window_synthetic_time}
	    }
			
	    \caption{Classification accuracies and execution times of SysFor~\cite{islam2011knowledge} and SysFor (Window) methods on Synthetic data set.}
	    \label{fig:justify_window_rf}
\end{figure}

\subsection{Incremental Tree Growing Mechanism of ADF}
\label{steps_framework}

To accommodate all scenarios as shown in Eq.~\ref{class_scenarios} and Fig.~\ref{fig:framework}, in ADF we propose an incremental tree growing mechanism, which is presented in Algorithm~\ref{algo_rfil}. 

At time $t^0$, ADF takes a batch data set ${D^{t^0}_{B_{train}}}$ (which is renamed as $D^B$ for simplicity) as input. It also takes a number of parameters as shown in Algorithm~\ref{algo_rfil}. Initially, the user-defined thresholds are initialised with their default values (see Section~\ref{dtrs}, Section~\ref{identify_concept_drift} and Section~\ref{justify_window}) and the forests and the statistics are set to $null$. For $D^B$, ADF makes use of Algorithm~\ref{algo_rfil} to build decision forests and calculate their statistics. ADF finds the best forest and recommends it for the user.

At time $t^i$, for $D^B$ (which is actually ${D^{t^i}_{B_{train}}}$) ADF updates the forests and statistics by using Algorithm~\ref{algo_rfil} and recommends the best forest for the user. ADF repeats this process for the remaining batches.

We now present the main steps of Algorithm~\ref{algo_rfil} as follows. In Step 1, we build $T^P$ (if it is $null$) by applying an existing decision forest algorithm such as RF~\cite{breiman2001random} on $D^B$. We also make a copy of $T^P$ to $T^A$. In Step 2, we first identify the perturbed leaves $F^P$ by using Eq.~\eqref{faultymatrix}. We then repair $T^P$ by using a procedure called RepairForest() which is presented in Algorithm~\ref{algo_repairRF}.  RepairForest() takes $T^P$, $D^B$, $L^P$, $F^P$ and $\theta$ as input and returns updated permanent forest ($T^{P\prime}$) and updated leaf statistics ($L^{P\prime}$) as output. In Step 3, we again find the perturbed leaves $F^A$ for $T^A$ by using Eq.~\ref{faultymatrix} and then calculate the percentage of perturbed leaves $\vartheta^A$ based on the process discussed in Section~\ref{dtrs}. If $\vartheta^A$ is less than or equal to $\theta$, we consider that $T^A$ is repairable. We then repair $T^A$ by using the procedure RepairForest() that returns updated active forest ($T^{A\prime}$) and updated leaf statistics ($L^{A\prime}$) as output.  

In Step 4, we update the window of batches $D^W$. Let $w=|D^W|$ be the number of batches that are currently stored in $D^W$. If $w$ is greater than or equal to $\lambda$, we first remove the oldest batch from $D^W$ and the add $D^B$ to $D^W$. In Step 5, we build $T^T$, $T_k^T$,$k=1,2,\ldots,M$ (if $T^T$ is null) by applying RF on $D^B$. Similar to Step 3, in Step 6 we repair $T^T$ if it is repairable. We find the perturbed leaves $F^T$ of $T^T$ and then calculate the percentage of perturbed leaves $\vartheta^T$. If $\vartheta^T$ is less than or equal to $\theta$, we consider that $T^T$ is repairable. We then repair $T^T$ by using procedure RepairForest() that returns updated temporary forest ($T^{T\prime}$) and updated leaf statistics ($L^{T\prime}$) as output. Finally, in step 7, Algorithm~\ref{algo_rfil} returns all updated classifiers, statistics and parameters.

\IncMargin{0.5em}
\begin{algorithm}
\SetKwBlock{StepZero}{Initialize:}{end}
\SetKwBlock{StepOne}{Step 1:}{end}
\SetKwBlock{StepTwo}{Step 2:}{end}
\SetKwBlock{StepThree}{Step 3:}{end}
\SetKwBlock{StepFour}{Step 4:}{end}
\SetKwBlock{StepFive}{Step 5:}{end}
\SetKwBlock{StepSix}{Step 6:}{end}
\SetKwBlock{StepSeven}{Step 7:}{end}

\SetKwInOut{Input}{Input}\SetKwInOut{Output}{Output}
{\scriptsize 
\Input{Batch data set $D^B$, Permanent Forest (PF) $T^P$, Active Forest (AF) $T^A$, Temporary Forest (TF) $T^R$, Statistics of PF leaves $L^P$, Statistics of AF leaves $L^A$, Statistics of TF leaves $L^T$, Reserve window $D^W$, Concept Drift Flag ($cdf$) Concept Drift Threshold ($\lambda$), Repairable Threshold ($\theta$), Error Tolerance Threshold ($\epsilon$) and Reserve Window Size ($\gamma$).
}
\Output{Updated $T^{P^\prime}$, $T^{A^\prime}$, $T^{T^\prime}$, $L^{P^\prime}$, $L^{A^\prime}$, $L^{T^\prime}$, $D^{W^\prime}$, and $cdf$.}
\BlankLine
\DontPrintSemicolon

\StepOne{
	\If {$T^P$ is $null$}{
Apply an existing decision forest algorithm such as RF~\cite{breiman2001random} on $D^B$ to build PF, $T_m^P$,$m=1,2,\ldots,M$; /* PF is built with M trees */\;
Set 	$T^A \leftarrow T^P$;\;	       
$Goto~Step~7$;\;	
}
 
}

\StepTwo{ 
Find perturbed leaves, $F^P \leftarrow FindPerturbedLeaves(D^B,T^P,\epsilon,L^P)$;  /* Find perturbed leaves in $T^P$ */\;
	$T^{P^\prime}, L^{P^\prime} \leftarrow RepairForest(T^P,D^B,L^P,F^P,\theta)$;\;
}

\StepThree{
Find perturbed leaves, $F^A \leftarrow FindPerturbedLeaves(D^B,T^A,\epsilon,L^A)$;  /* Find faulty leaves in $T^A$ */\;
Calculate, $\vartheta^A\leftarrow CalculatePerturbedLeavesRatio(F^A)$;  /* Calculate the percentage of faulty leaves in $T^A$ */\;
\If{$\vartheta^A \leq \theta$}{ 
$T^{A^\prime},L^{A^\prime}  \leftarrow RepairForest(T^A,D^B,L^A,F^A,\theta)$;\;
$Goto~Step~7$;
}
}
\StepFour{	
	\If{$|D^W|\geq\gamma$}{
			Find the oldest batch in $D^W$ and set to $D^O$;\;
			$D^W \leftarrow D^W \cap D^O$;\;
}
$D^{W^\prime} \leftarrow D^W \cup D^B$;\;
$cdf \leftarrow cdf+1$;\;
}
\StepFive{
\If{$T^T$ is $null$}{
Apply RF on $D^{W^\prime}$ to build TF, $T_m^T$, $m=1,2,\ldots,M$;\;
$Goto~Step~7$;	
}
}
\StepSix{
Find perturbed leaves, $F^T \leftarrow FindPerturbedLeaves(D^B,T^T,\epsilon,L^T)$;  /* Find faulty leaves in $T^R$ */\;
Calculate, $\vartheta^R \leftarrow CalculatePerturbedLeavesRatio(F^T)$;  /* Calculate the percentage of perturbed leaves in $T^T$ */\;
 \If{$\vartheta^T \leq \theta$}{ 
 $T^{T^\prime}, L^{T^\prime} \leftarrow RepairForest(T^T,D^B,L^T,F^T,\theta)$;\;
}
\Else{
 Apply RF on $D^{W^\prime}$ to build TF, $T_m^T$, $m=1,2,..,M$;\; 
\If{$cdf>\lambda$}{
		$T^{A^\prime} \leftarrow T^{T^\prime}$; $T^{T^\prime} \leftarrow \emptyset$; $cdf \leftarrow 0$; $L^{A^\prime} \leftarrow L^{T^\prime}$;\;
	}
}
}
\StepSeven{
	Return $T^{P^\prime}$, $T^{A^\prime}$, $T^{T^\prime}$, $L^{P^\prime}$, $L^{A^\prime}$, $L^{T^\prime}$, $D^{W^\prime}$ and $cdf$.
}
}
\caption{LearnADF()}\label{algo_rfil}
\end{algorithm}\DecMargin{0.5em}

\IncMargin{0.5em}
\begin{algorithm}
\SetAlgoLined
\SetKwData{Left}{left}\SetKwData{This}{this}\SetKwData{Up}{up}

\SetKwBlock{StepOne}{Step 1:}{end}
\SetKwBlock{StepTwo}{Step 2:}{end}
\SetKwInOut{Input}{Input}\SetKwInOut{Output}{Output}
{\scriptsize 
\Input{Decision forest, $T$, Batch data set, $D^B$, Leaves statistics, $L$, Perturbed leaves $F$ and Repairable threshold $\theta$.
}
\Output{Updated decision forest $T^\prime$, updated leaves statistics, $L^\prime$.}
\BlankLine
\DontPrintSemicolon

\StepOne{
	Set $T^\prime \leftarrow \emptyset$;\;
	Set $L^\prime \leftarrow \emptyset$;\;
	\For{$k=1$ \KwTo $M$}{
		calculate, $\vartheta \leftarrow CalculatePerturbedLeavesRatio(F_k)$; /*Calculate the percentages of perturbed leaves of the $k^{th}$-tree, $T_k$.*/\;
		\If{$\vartheta > \theta$}{
			$T^\prime_k \leftarrow iSAT(T_k,D^B)$; /*Repair $T_k$ by using our proposed improved separating axis theorem (iSAT).*/ \;
			\For{each leaf, $l \in T^\prime_i$}{
			    Set $D^l \leftarrow  FindRecords(T_k,D^B,l)$; \;
			   \If{$isPure(D^l)==false \& isPerturbed(l)\& SizeOf(D^l)> MinLeafSize$}{
			Expand leaf $l$ using Entropy based splitting strategy;\;
		}
		}
		}
		calculate, $L^\prime_k \leftarrow UpdateLeafStatistics(T^\prime_k,D^B,L)$;\;
		$L^\prime \leftarrow L^\prime \cup L^\prime_k$
	    $T^\prime \leftarrow T^\prime \cup T^\prime_k$
	}
}
\StepTwo{ 
	Return $T^\prime$;
}
}
\caption{Procedure RepairForest()}
\label{algo_repairRF}
\end{algorithm}
\DecMargin{0.5em}

\IncMargin{0.5em}
\begin{algorithm}
\SetAlgoLined
\SetKwData{Left}{left}\SetKwData{This}{this}\SetKwData{Up}{up}

\SetKwBlock{StepOne}{Step 1:}{end}
\SetKwBlock{StepTwo}{Step 2:}{end}
\SetKwBlock{StepThree}{Step 3:}{end}
\SetKwInOut{Input}{Input}\SetKwInOut{Output}{Output}
{\scriptsize 
\Input{Decision tree, $t$, Batch data set, $D^B$}
\Output{Updated decision tree, $t^\prime$}
\BlankLine
\DontPrintSemicolon
\StepOne{
	set $LB^t \leftarrow min(t)$;/*Find lower boundary of $t$ for all numerical attributes. */\; 
	set $UB^t \leftarrow max(t)$; /*Find upper boundary of $t$ for all numerical attributes. */\; 
	set $LB^B \leftarrow min(D^B)$; /*Find lower boundary of $D^B$ for all numerical attributes. */\; 
	set $UB^B \leftarrow max(D^B)$; /*Find upper boundary of $D^B$ for all numerical attributes. */\; 
	set $Dist_1 \leftarrow LB^B - UB^t$; $Dist_2 \leftarrow LB^t - UB^B$; \;
	set $Max_1 \leftarrow max(Dist_1)$; $Max_2 \leftarrow max(Dist_2)$; \;
}
\StepTwo{
/*If there is no overlapping between $t$ and $D^B$. */\; 
\If{($Max_1 > 0 \& Max_1 \geq Max_2$) OR ($Max_2 > 0 \& Max_2 \geq Max_1$) }{
\If{$Max_1 > 0 \& Max_1 \geq Max_2$}{
        Set $splitAttr \leftarrow findAttributeWithMaxDistance(Dist_1,Max_1)$; \;
		set $SplitValue \leftarrow (LB^B[splitAttr]+UB^t[splitAttr])/2$;\;
		Create a node, $node \leftarrow CreateNode(splitAttr,SplitValue)$;\;
		Add $t$ as the left child of $t^\prime$ ;\;
		Add $node$ as the right child of $t^\prime$;\;
}
\ElseIf{$Max_2 > 0 \& Max_2 \geq Max_1$}{
        Set $splitAttr \leftarrow findAttributeWithMaxDistance(Dist_2,Max_2)$;\;
		set $SplitValue \leftarrow (LB^t[splitAttr]+UB^B[splitAttr])/2$;\;
		Create a node, $node \leftarrow CreateNode(splitAttr,SplitValue)$;\;
		Add $t$ as the right child of $t^\prime$ ;\;
		Add $node$ as the left child of $t^\prime$;\;
}
}
/*If there is overlapping exists between $t$ and $D^B$. */\; 
\Else{
	Update $Dist_1 \leftarrow UB^B - UB^t$;\; 
	Update $Max_1 \leftarrow max(Dist_1)$;\; 
	\If{$Max_1 > 0$}{
	    Set $splitAttr \leftarrow findAttributeWithMaxDistance(Dist_1,Max_1)$; \;
	    set $SplitValue \leftarrow UB^t[splitAttr]$;\;
		Create a node, $node \leftarrow CreateNode(splitAttr,SplitValue)$;\;
		Add $t$ as the left child of $t^\prime$ ;\;
		Add $node$ as the right child of $t^\prime$;\;
        }
    Update $Dist_2 \leftarrow LB^t - LB^B$;\;
	Update $Max_2 \leftarrow max(Dist_2)$; \;
	\If{$Max_2 > 0$}{
	    Set $splitAttr \leftarrow findAttributeWithMaxDistance(Dist_2,Max_2)$;\;
	    set $SplitValue \leftarrow LB^t[splitAttr]$;\;
		Create a node, $node \leftarrow CreateNode(splitAttr,SplitValue)$;\;
		Add $t$ as the right child of $t^\prime$ ;\;
		Add $node$ as the left child of $t^\prime$;\;	
        }
}

}

\StepThree{
Return updated decision tree, $t^\prime$;
}
}
\caption{Procedure iSAT()}\label{algo_iSAT}
\end{algorithm}\DecMargin{0.5em}

\subsection{Complexity Analysis of ADF}
\label{Complexity}

We now analyse the computational complexity of ADF. We consider a batch data set with $n$ records, and $m$ attributes. We also consider that the ensemble size is $M$ where a tree can have $l$ leaves. Although, ADF iteratively updates the trees based on each training batch data using Algorithm~\ref{algo_rfil}, we present a complexity analysis of the algorithm for an iteration as follows. 

In Step 1 of Algorithm~\ref{algo_rfil}, we build a decision forest by using an existing algorithm such as RF~\cite{breiman2001random} which has a complexity $O(Mnm^2)$~\cite{SuZhang2006}. In Step 2, we find the number of perturbed leaves using Eq.~\ref{faultymatrix} which has a complexity $O(Mnml)$. In this step, we also repair the trees one by one by using the RepairForest procedure (see Algorithm~\ref{algo_repairRF}). For $M$ trees, the complexity of the RepairForest() is $O(M(nm+\theta nm^2))\approx O(Mnm^2)$. Thus, the complexity of Step 2 of Algorithm~\ref{algo_rfil} is $O(Mnml+Mnm^2)\approx O(Mnm^2)$ for a small $l$. Step 3 of the algorithm has the same complexity.

In Step 4 of Algorithm~\ref{algo_rfil}, we update the window of batches that has a complexity $O(\gamma{n})$. Similar to Step 1, Step 5 has a complexity $O(M{n}m^2)$ for a small $\gamma$. In Step 6, we either repair trees or build new trees. The complexity of these operations is $O(Mnm^2)$. Thus, the overall complexity of Algorithm~\ref{algo_rfil} is $O(Mnm^2)$. This variant of ADF is known as ADF-R. 

ADF can also build a decision forest by using a decision tree algorithm such as the Hoeffding Tree (HT) algorithm~\cite{domingos2000mining}. The complexity of the HT algorithm is $O(nmlcv)$~\cite{gomes2017adaptive}, where $c$ is the number of class values and $v$ is the maximum domain size of an attribute. Thus, for $M$ trees, ADF requires a complexity $O(Mnmlcv)$. This variant of ADF is known as ADF-H. Typically, $c$, $v$ and $l$ values are very small, especially compared to $n$. Therefore, the complexity of ADF-H is $O(Mnm)$. We present the complexities of ADF variants, CIRF, RF and ARF in Table~\ref{tab:Complexity_analysis}.

 \begin{table}[ht!]
	\small
	\centering
	\caption{Complexity analysis.}
	\renewcommand\tabcolsep{3pt} 
	\begin{tabular}{lr|lr}
	\toprule
		Proposed Method&Computational complexity&Existing Method&Computational complexity\\
		\midrule
		  ADF-R&$O(Mnm^2)$&RF&$O(Mnm^2)$~\cite{SuZhang2006}\\
		  &&CIRF&$O(Mnm^2)$~\cite{hu2018novel}\\
		  ADF-H&$O(Mnm)$&ARF&$O(Mnm(log(m)))$~\cite{gomes2017adaptive}\\
		\bottomrule
	\end{tabular}
	\label{tab:Complexity_analysis}
\end{table}

\section{Experimental Results and Discussion}
\label{Experimental_Result}

We carry out a set of experiments to evaluate our proposed framework ADF. We compare the performance of ADF with eight state-of-the-art machine learning and incremental learning techniques, comprising two non-incremental forest algorithms (namely RF~\cite{breiman2001random} and SysFor~\cite{islam2011knowledge}), two incremental tree algorithms (namely Hoeffding Tree (HT)~\cite{domingos2000mining} and Hoeffding Adaptive Tree (HAT)~\cite{bifet2009adaptive} and four incremental forest algorithms (namely LeveragingBag~\cite{bifet2010leveraging}, OzaBag~\cite{oza2005online}, CIRF~\cite{hu2018novel} and ARF~\cite{gomes2017adaptive}).  

We implement ADF in the Java programming language using the Massive Online Analysis (MOA) API~\cite{MOA_API}. The source code of ADF is available at GitHub (\url{https://github.com/grahman20/ADF}). We test the ADF framework by using one of the three techniques: RF~\cite{breiman2001random}, HT~\cite{domingos2000mining} and SysFor~\cite{islam2011knowledge}, and thereby obtain three variants called ADF-R, ADF-H and ADF-S, respectively. 
We also implement an existing technique called CIRF~\cite{hu2018novel}. For implementing the RF and SysFor, we use the Java code from the Weka platform~\cite{witten2016data}. All other methods are already available in the MOA framework. We use the default settings of the methods while running the experiment. 

\subsection{Data Sets}
\label{Data_Sets}

We apply the incremental learning techniques on five real data sets and one synthetic data set, as shown in Table~\ref{tab:Datasets}. The real data sets are available publicly in the UCI Repository~\cite{Frank+Asuncion:2010}. 

\begin{table}[ht!]
	\small
	\centering
	\caption{Data sets at a glance.}
	\renewcommand\tabcolsep{3pt} 
	\begin{tabular}{lrcccl}
	\toprule
		Data set&Records& Attributes&Number of class values & Classification accuracy & Data source\\
		\midrule
		  LDPA&164,860&6&11&56\%&UCI\\
		  UIFWA&149,332&5&22&40\%&UCI\\
			EB&45,781&5&31&66\%&UCI\\
			AReM&42,239&7&7&63\%&UCI\\
			Avila&20,867&11&12&51\%&UCI\\
			House&300,000&10&7&80\%&Synthetic\\
		\bottomrule
	\end{tabular}
	\label{tab:Datasets}
\end{table}

\subsection{Simulation of Training and Test Batch Data Sets}
\label{Experimental_batch_creation}

For each data set, we artificially create 34 training and test batch data sets in which we implement the five scenarios, namely SKC, MKC, SUC, MUC and MKUC as discussed in Section~\ref{Our_Framework} and presented in Eq.~\eqref{class_scenarios} and in Fig.~\ref{fig:framework}. The process of simulating scenarios and batches is summarized in Table~\ref{tab:simulation_batch_data}.

We first divide a data set into two sub data sets, where each sub data set contains approximately half of the total records of the data set and approximately half of the class values. From the first sub data set, we create 3 batches (having the numbers from 9 to 11) that follow the SKC scenario in which each batch contains records having a single class value. Similarly, from the second sub data set, we create 3 batches (having the numbers from 12 to 14) which follow the SUC scenario as shown in Table~\ref{tab:simulation_batch_data}.

We then build two decision trees: $T_1$ and $T_2$ by applying SysFor~\cite{islam2011knowledge} on the first and second sub data sets, respectfully. Using $T_1$ and $T_2$, we then create the remaining 28 batches where the source of records of each batch is described in the third column of Table~\ref{tab:simulation_batch_data}. For each batch, we create training and test batch data sets, where the test batch contains 20\% of records in which 10\% of records are taken from the current batch and the remaining 10\% of records are taken from the previous two batches. Note that all records in the training and test batches are chosen randomly. 
 
\begin{table}[ht!]
	\small
	\centering
	\caption{Simulation of batch data sets.}
	\renewcommand\tabcolsep{3pt} 
	\begin{tabular}{lcl}
	\toprule
		Scenario&Batches& Source of records\\
		\midrule
		  MKC-1&1-4&
			\begin{tabular}{@{}l@{}}25\% of records are taken from the largest two leaves of $T_1$ and \\ 75\% of records are taken from the remaining leaves of $T_1$
			\end{tabular}
			 \\
			\cline{2-3}
		  MKC-2&5-8&
			\begin{tabular}{@{}l@{}}75\% of records are taken from the largest two leaves of $T_1$ and \\ 25\% of records are taken from the remaining leaves of $T_1$
			\end{tabular}
			 \\			
			\midrule
			SKC&9-11& Each batch contains records (from the first sub data set) having a single class value\\
		  \midrule
			SUC&12-14& Each batch contains records (from the second sub data set) having a single class value\\
			\midrule
		  MUC-1&15-18&
			\begin{tabular}{@{}l@{}}25\% of records are taken from the largest two leaves of $T_2$ and \\ 75\% of records are taken from the remaining leaves of $T_2$
			\end{tabular}
			 \\
			\cline{2-3}
		  MUC-1&19-22&
			\begin{tabular}{@{}l@{}}75\% of records are taken from the largest two leaves of $T_2$ and \\ 25\% of records are taken from the remaining leaves of $T_2$
			\end{tabular}
			 \\
			\midrule
		  MKUC-1&23-26&75\% of records are taken from $T_1$ and 25\% of records are taken from $T_2$\\
		  MKUC-2&27-30&25\% of records are taken from $T_1$ and 75\% of records are taken from $T_2$\\
			MKUC-3&31-34&50\% of records are taken from $T_1$ and 50\% of records are taken from $T_2$\\			
	\bottomrule
	\end{tabular}
	\label{tab:simulation_batch_data}
\end{table}

\subsection{Experimental Settings}
\label{Experimental_settings}

In our experiment, we use a number of parameters for ADF variants and existing techniques. Most of the techniques use two common parameters, namely ensemble size and minimum leaf size (or grace period). The ensemble size is set to 10 and the minimum leaf size for large data sets (i.e. the size of the data set is greater than 100000) is set to 100; otherwise it is set to 20. For ADF, we set the default values for $\lambda$, $\theta$, $\epsilon$, and $\gamma$ at 3,0.4, 0.02, and 3, respectively. Moreover, we use the majority voting~\cite{jamali2011majority} method to calculate the final output of a decision forest.

\subsection{Detail Experimental Results and Performance Analysis}
\label{Experimental_detail_results}

We present the performances of ADF variants (ADF-H, ADF-R and ADF-S) and eight existing techniques on 34 batches that are created from the LDPA data set in Table~\ref{tab:LDPA_details_accuracy}. In the case of incremental methods, a classifier is built by applying a method on the first batch training data and the classifier is repaired and updated incrementally for the remaining batches. However, for non-incremental methods, we rebuild the classifier for each batch training data. For both types of methods, we obtain a classier for batch training data ${D^{t^i}_{B_{train}}}$. We then use the classifier to classify batch test data $D^{t^i}_{B_{test}}$ and obtain classification accuracy. Bold values in the table indicate the best results. Out of 34 batches, ADF-S performs the best in 25 batches.

For each method, we also calculate the average classification accuracy which is shown in the last row of Table~\ref{tab:LDPA_details_accuracy}. The average classification accuracies of ADF-R and ADF-S are 0.640 and 0.682, respectively. From the experimental results, it is clear that ADF-R and ADF-S outperform other methods on the LDPA data set.

\begin{table*}[!ht]
\tiny
	\centering
	\caption{Classification performances of ADF variants and existing methods on 34 batches of LDPA data set.}
	\renewcommand\arraystretch{1.2} 
	\renewcommand\tabcolsep{0pt} 
	\begin{tabular*}{\textwidth}{@{\extracolsep{\fill}}lcccccccccccccc}
		\toprule
		\multirow{2}*{Scenario} & \multirow{2}*{Batch No.} & \multicolumn{2}{c}{Batch Size} & \multicolumn{2}{c}{Non-incremental Forest}& \multicolumn{2}{c}{Incremental Tree}& \multicolumn{7}{c}{Incremental Forest}\\  
	\cline{3-4}\cline{5-6}\cline{7-8}\cline{9-15}
		& &Training Data&Test Data&SysFor&RF&HT&HAT&LeveragingBag&OzaBag&CIRF&ARF&ADF-H&ADF-R&ADF-S\\
		\midrule
\multirow{4}*{MKC-1}&1&2613&803&\textbf{0.707}&0.690&0.247&0.247&0.247&0.247&\textbf{0.707}&0.262&0.690&0.690&\textbf{0.707}\\
&2&3216&802&\textbf{0.766}&0.717&0.069&0.069&0.070&0.069&0.743&0.080&0.721&0.722&0.749\\
&3&3205&801&\textbf{0.758}&0.730&0.213&0.705&0.664&0.642&0.742&0.690&0.720&0.727&0.740\\
&4&3194&801&\textbf{0.782}&0.754&0.738&0.738&0.705&0.705&0.744&0.729&0.738&0.742&0.770\\
	\cline{1-15}
\multirow{4}*{MKC-2}&5&3391&845&\textbf{0.783}&0.773&0.749&0.749&0.737&0.753&0.756&0.753&0.749&0.748&0.750\\
&6&3391&845&\textbf{0.807}&0.806&0.263&0.175&0.174&0.392&0.793&0.180&0.775&0.776&0.786\\
&7&3360&845&\textbf{0.793}&0.786&0.199&0.176&0.180&0.299&0.780&0.196&0.759&0.761&0.759\\
&8&3329&845&\textbf{0.808}&0.792&0.683&0.142&0.062&0.657&0.798&0.063&0.765&0.765&0.774\\
	\cline{1-15}
\multirow{3}*{SKC}&9&4001&968&0.885&0.885&0.885&0.885&0.885&0.861&0.882&0.885&0.885&\textbf{0.886}&0.883\\
&10&4001&968&0.551&0.551&0.551&0.551&0.551&0.550&0.614&0.551&0.476&0.699&\textbf{0.745}\\
&11&1898&484&0.500&0.500&0.500&0.500&0.500&0.502&0.310&0.500&0.625&0.635&\textbf{0.644}\\
	\cline{1-15}
\multirow{3}*{SUC}&12&4080&968&0.500&0.500&0.500&0.500&0.500&0.500&0.210&0.500&0.458&0.600&\textbf{0.655}\\
&13&296&169&0.503&0.503&0.503&0.503&0.503&0.503&0.271&0.503&0.541&0.611&\textbf{0.612}\\
&14&3904&968&0.500&0.500&0.500&0.500&0.500&0.500&0.276&0.500&\textbf{0.692}&0.590&0.590\\
	\cline{1-15}
\multirow{4}*{MUC-1}&15&3685&920&0.450&0.340&0.180&0.251&0.274&0.180&0.213&0.302&0.398&0.416&\textbf{0.489}\\
&16&3685&920&0.549&0.505&0.495&0.297&0.297&0.490&0.340&0.272&0.495&0.641&\textbf{0.660}\\
&17&3682&920&0.572&0.502&0.067&0.096&0.105&0.072&0.404&0.097&0.386&0.514&\textbf{0.638}\\
&18&3679&920&0.524&0.501&0.316&0.172&0.150&0.316&0.395&0.155&0.368&0.517&\textbf{0.629}\\
	\cline{1-15}
\multirow{4}*{MUC-2}&19&3746&933&0.660&0.548&0.403&0.251&0.170&0.403&0.467&0.195&0.432&0.592&\textbf{0.702}\\
&20&3746&933&0.662&0.606&0.027&0.098&0.099&0.027&0.525&0.107&0.414&0.599&\textbf{0.712}\\
&21&3685&933&0.663&0.581&0.323&0.297&0.273&0.323&0.531&0.287&0.402&0.640&\textbf{0.735}\\
&22&3624&933&0.554&0.544&0.066&0.194&0.136&0.066&0.466&0.134&0.407&0.583&\textbf{0.680}\\
	\cline{1-15}
\multirow{4}*{MKUC-1}&23&4714&1177&0.466&0.371&0.295&0.295&0.302&0.295&0.453&0.299&0.394&0.604&\textbf{0.641}\\
&24&4714&1177&0.556&0.472&0.427&0.427&0.130&0.427&0.452&0.134&0.469&0.637&\textbf{0.659}\\
&25&4714&1177&0.520&0.577&0.104&0.104&0.110&0.104&0.500&0.098&0.579&0.663&\textbf{0.693}\\
&26&4714&1177&0.605&0.578&0.557&0.557&0.524&0.557&0.488&0.555&0.590&0.664&\textbf{0.692}\\
	\cline{1-15}
\multirow{4}*{MKUC-2}&27&4704&1176&0.554&0.418&0.177&0.177&0.253&0.177&0.508&0.221&0.413&0.603&\textbf{0.671}\\
&28&4704&1176&0.561&0.429&0.026&0.026&0.103&0.026&0.530&0.131&0.420&0.599&\textbf{0.655}\\
&29&4691&1176&0.558&0.442&0.016&0.110&0.119&0.016&0.431&0.102&0.329&0.557&\textbf{0.625}\\
&30&4683&1176&0.554&0.389&0.168&0.201&0.159&0.168&0.500&0.143&0.317&0.566&\textbf{0.607}\\
	\cline{1-15}
\multirow{4}*{MKUC-3}&31&4920&1229&0.579&0.407&0.234&0.150&0.136&0.234&0.530&0.138&0.362&0.599&\textbf{0.632}\\
&32&4811&1209&0.560&0.392&0.285&0.037&0.070&0.285&0.526&0.074&0.383&0.596&\textbf{0.634}\\
&33&5338&1256&0.569&0.450&0.171&0.160&0.104&0.171&0.522&0.102&0.410&0.629&\textbf{0.651}\\
&34&5814&1292&0.557&0.484&0.106&0.135&0.098&0.106&0.480&0.101&0.385&0.594&\textbf{0.629}\\
	\cline{1-15}
	\multicolumn{4}{c}{Average}&0.615&0.559&0.325&0.308&0.291&0.342&0.526&0.295&0.528&0.640&\textbf{0.682}\\
		\bottomrule
	\end{tabular*}
	\label{tab:LDPA_details_accuracy}
\end{table*}

Similar to LDPA, for the remaining five data sets ADF variants achieve higher classification accuracies than other existing techniques. Due to space limitation, we present the average classification accuracies (see Table~\ref{tab:overall_avg_accuracy}) and execution times (see Table~\ref{tab:overall_avg_time}) of the ADF variants and eight existing techniques on six data sets. Bold values in the tables indicate the best results. From Table~\ref{tab:overall_avg_accuracy}, we can see that ADF-R and ADF-S outperform the other techniques in terms of classification accuracy for all data sets. However, while ADF variants perform better than other techniques, the variants take slightly more time than the ARF, CIRF, RF and SysFor. As accurate classification is the primary goal in real applications, for handling dynamic changes in both class and data distribution, ADF-R can be recommended over the eight state-of-the-art techniques by considering the classification accuracy and execution time.

\begin{table*}[!ht]
\tiny
	\centering
	\caption{Overall classification accuracy of ADF variants and other existing methods on all data sets.}
	\renewcommand\arraystretch{1.2} 
	\renewcommand\tabcolsep{0pt} 
	\begin{tabular*}{\textwidth}{@{\extracolsep{\fill}}lccccccccccc}
		\toprule
		\multirow{2}*{Data set} & \multicolumn{2}{c}{Non-incremental Forest}& \multicolumn{2}{c}{Incremental Tree}& \multicolumn{7}{c}{Incremental Forest}\\  
	\cline{2-3}\cline{4-5}\cline{6-12}
		&SysFor&RF&HT&HAT&LeveragingBag&OzaBag&CIRF&ARF&ADF-H&ADF-R&ADF-S\\
		\midrule
		LDPA&0.615&0.559&0.325&0.308&0.291&0.342&0.526&0.295&0.528&0.640&\textbf{0.682}\\
UIFWA&0.484&0.504&0.186&0.191&0.195&0.182&0.393&0.205&0.416&\textbf{0.591}&0.550\\
EB&0.627&0.619&0.407&0.315&0.343&0.409&0.629&0.516&0.626&\textbf{0.722}&0.705\\
AReM&0.714&0.729&0.327&0.311&0.285&0.346&0.650&0.289&0.522&\textbf{0.838}&0.818\\
Avila&0.685&0.719&0.382&0.376&0.328&0.395&0.643&0.350&0.582&0.788&\textbf{0.825}\\
House&0.801&0.802&0.390&0.339&0.354&0.402&0.720&0.351&0.853&0.860&\textbf{0.867}\\
		\bottomrule
	\end{tabular*}
	\label{tab:overall_avg_accuracy}
\end{table*}

\begin{table*}[!ht]
\tiny
	\centering
	\caption{Overall execution time (ms) of ADF variants and other existing methods on all data sets.}
	\renewcommand\arraystretch{1.2} 
	\renewcommand\tabcolsep{0pt} 
	\begin{tabular*}{\textwidth}{@{\extracolsep{\fill}}lccccccccccc}
		\toprule
		\multirow{2}*{Data set} & \multicolumn{2}{c}{Non-incremental Forest}& \multicolumn{2}{c}{Incremental Tree}& \multicolumn{7}{c}{Incremental Forest}\\  
	\cline{2-3}\cline{4-5}\cline{6-12}
		&SysFor&RF&HT&HAT&LeveragingBag&OzaBag&CIRF&ARF&ADF-H&ADF-R&ADF-S\\
		\midrule
LDPA&2651&1791&\textbf{20}&30&281&103&2062&938&1415&2901&3589\\
UIFWA&1591&1584&\textbf{70}&185&369&657&1785&1149&1226&3287&3528\\
EB&511&500&\textbf{44}&45&102&82&444&710&960&1357&1426\\
AReM&790&427&\textbf{62}&103&363&137&929&477&654&959&1011\\
Avila&288&222&\textbf{28}&37&157&70&808&424&506&643&917\\
House&2673&2829&\textbf{27}&65&350&167&2292&1211&2135&3184&3395\\
		\bottomrule
	\end{tabular*}
	\label{tab:overall_avg_time}
\end{table*}

To demonstrate the effectiveness of ADF framework, we also carry out experimentation by rearranging the scenarios of the batches on LDPA and UIFWA. In the new arrangement, we have 28 batches where the training and test batches are created randomly. Table~\ref{tab:LDPA_details_accuracy_batchrearranged} presents the details about the rearrangement of scenarios, and classification accuracies of ADF variants and existing techniques on the LDPA data set. Bold values in the table indicate that ADF-S performs the best in all batches.

\begin{table*}[!ht]
\tiny
	\centering
	\caption{Accuracies of ADF variants and existing methods by rearranging the scenarios of the batches on LDPA data set.}
	\renewcommand\arraystretch{1.2} 
	\renewcommand\tabcolsep{0pt} 
	\begin{tabular*}{\textwidth}{@{\extracolsep{\fill}}lcccccccccccccc}
		\toprule
		\multirow{2}*{Scenario} & \multirow{2}*{Batch No.} & \multicolumn{2}{c}{Batch Size} & \multicolumn{2}{c}{Non-incremental Forest}& \multicolumn{2}{c}{Incremental Tree}& \multicolumn{7}{c}{Incremental Forest}\\  
	\cline{3-4}\cline{5-6}\cline{7-8}\cline{9-15}
		& &Training Data&Test Data&SysFor&RF&HT&HAT&LeveragingBag&OzaBag&CIRF&ARF&ADF-H&ADF-R&ADF-S\\
		\midrule
		\multirow{4}*{MKC-1}&1&2981&917&\textbf{0.662}&0.614&0.325&0.325&0.320&0.310&\textbf{0.662}&0.323&0.578&0.614&\textbf{0.662}\\
&2&3669&916&\textbf{0.689}&0.624&0.107&0.107&0.180&0.218&0.664&0.175&0.592&0.652&\textbf{0.689}\\
&3&3641&916&0.663&0.634&0.162&0.097&0.121&0.154&0.669&0.119&0.615&0.652&\textbf{0.691}\\
&4&3613&916&0.638&0.549&0.393&0.540&0.512&0.400&0.632&0.550&0.541&0.607&\textbf{0.662}\\
\cline{1-15}
\multirow{4}*{MKC-2}&5&4130&1029&0.739&0.704&0.335&0.227&0.230&0.333&0.731&0.225&0.702&0.725&\textbf{0.758}\\
&6&4130&1029&0.765&0.758&0.732&0.756&0.756&0.703&0.781&0.750&0.756&0.765&\textbf{0.786}\\
&7&4047&1029&0.794&0.777&0.511&0.052&0.056&0.511&0.811&0.055&0.777&0.791&\textbf{0.814}\\
&8&3964&1029&0.794&0.779&0.751&0.779&0.779&0.724&0.796&0.779&0.780&0.787&\textbf{0.810}\\
\cline{1-15}
\multirow{4}*{MKUC-1}&9&5457&1361&0.700&0.670&0.654&0.654&0.647&0.642&0.691&0.630&0.656&0.695&\textbf{0.724}\\
&10&5457&1361&0.652&0.604&0.079&0.061&0.062&0.063&0.622&0.065&0.588&0.646&\textbf{0.660}\\
&11&5457&1361&0.616&0.555&0.151&0.096&0.096&0.140&0.552&0.086&0.531&0.589&\textbf{0.644}\\
&12&5457&1361&0.531&0.588&0.042&0.191&0.181&0.150&0.597&0.182&0.553&0.620&\textbf{0.652}\\
\cline{1-15}
\multirow{4}*{MKUC-2}&13&5456&1361&0.526&0.482&0.172&0.159&0.080&0.133&0.491&0.097&0.357&0.478&\textbf{0.589}\\
&14&5456&1361&0.582&0.445&0.258&0.202&0.217&0.247&0.453&0.238&0.266&0.525&\textbf{0.600}\\
&15&5534&1361&0.580&0.519&0.155&0.244&0.055&0.149&0.425&0.265&0.285&0.525&\textbf{0.609}\\
&16&5612&1361&0.589&0.506&0.179&0.245&0.051&0.195&0.405&0.259&0.300&0.544&\textbf{0.598}\\
\cline{1-15}
\multirow{4}*{MUC-1}&17&4207&1049&0.557&0.459&0.214&0.046&0.079&0.187&0.410&0.269&0.341&0.582&\textbf{0.605}\\
&18&4207&1049&0.594&0.470&0.331&0.097&0.152&0.291&0.396&0.131&0.380&0.583&\textbf{0.605}\\
&19&4205&1049&0.520&0.501&0.362&0.238&0.310&0.308&0.418&0.384&0.425&0.609&\textbf{0.619}\\
&20&4203&1049&0.575&0.447&0.417&0.421&0.409&0.416&0.404&0.407&0.410&0.624&\textbf{0.628}\\
\cline{1-15}
\multirow{4}*{MUC-2}&21&4230&1056&0.631&0.618&0.342&0.303&0.086&0.331&0.517&0.080&0.398&0.634&\textbf{0.667}\\
&22&4230&1056&0.696&0.599&0.277&0.325&0.147&0.270&0.564&0.149&0.402&0.666&\textbf{0.726}\\
&23&4136&1056&0.632&0.563&0.334&0.186&0.252&0.363&0.578&0.272&0.386&0.677&\textbf{0.728}\\
&24&4047&1056&0.635&0.656&0.349&0.194&0.325&0.323&0.580&0.281&0.366&0.671&\textbf{0.728}\\
\cline{1-15}
\multirow{4}*{MKUC-3}&25&5740&1433&0.581&0.479&0.266&0.278&0.227&0.262&0.518&0.207&0.293&0.617&\textbf{0.639}\\
&26&5618&1413&0.580&0.402&0.233&0.246&0.137&0.238&0.524&0.118&0.346&0.620&\textbf{0.640}\\
&27&6242&1469&0.543&0.447&0.280&0.212&0.170&0.251&0.500&0.191&0.372&0.611&\textbf{0.630}\\
&28&6817&1513&0.538&0.484&0.231&0.295&0.264&0.293&0.499&0.249&0.398&0.601&\textbf{0.606}\\
	\cline{1-15}
	\multicolumn{4}{c}{Average}&0.629&0.569&0.309&0.271&0.246&0.307&0.568&0.269&0.478&0.633&\textbf{0.670}\\
		\bottomrule
	\end{tabular*}
	\label{tab:LDPA_details_accuracy_batchrearranged}
\end{table*}

By rearranging the scenarios, we present the overall accuracies of ADF variants and other existing methods on the LDPA and UIFWA data sets in Table~\ref{tab:overall_avg_accuracy_batchrearranged}. It is clear that ADF-R and ADF-S also outperform other techniques in terms of classification accuracy after rearranging the scenarios of the batches.

\begin{table*}[!ht]
\tiny
	\centering
	\caption{Overall accuracies of ADF variants and other methods on LDPA and UIFWA data sets by rearranging the scenarios.}
	\renewcommand\arraystretch{1.2} 
	\renewcommand\tabcolsep{0pt} 
	\begin{tabular*}{\textwidth}{@{\extracolsep{\fill}}lccccccccccc}
		\toprule
		\multirow{2}*{Data set} & \multicolumn{2}{c}{Non-incremental Forest}& \multicolumn{2}{c}{Incremental Tree}& \multicolumn{7}{c}{Incremental Forest}\\  
	\cline{2-3}\cline{4-5}\cline{6-12}
		&SysFor&RF&HT&HAT&LeveragingBag&OzaBag&CIRF&ARF&ADF-H&ADF-R&ADF-S\\
		\midrule
LDPA&0.629&0.569&0.309&0.271&0.246&0.307&0.568&0.269&0.478&0.633&\textbf{0.670}\\
UIFWA&0.493&0.518&0.114&0.092&0.112&0.119&0.409&0.117&0.413&\textbf{0.565}&0.540\\
		\bottomrule
	\end{tabular*}
	\label{tab:overall_avg_accuracy_batchrearranged}
\end{table*}

\subsection{Statistical Analysis of the Experimental Results}
\label{Experimental_stat_analysis}

We now analyze the results by using a statistical non-parametric sign test~\cite{mason1994statistics} and Nemenyi~\cite{demvsar2006statistical} test for all 34 batches to evaluate the statistical significance of the superiority of ADF variants over the existing methods. 

We carry out a sign test on the results of ADF-S with other methods (one by one) at the right-tailed by considering significance level $\alpha=0.025$ (i.e. 97.5\% significance level) as shown in Fig.~\ref{fig:sign_all_adf_s}. In the figure, each bar represents the z-value of the comparison between ADF-S and an existing method, whereas the line represents the z-ref value (which is obtained from a table~\cite{mason1994statistics}). The sign test results (see Fig.~\ref{fig:sign_all_adf_s}) indicate that ADF-S performs significantly better than the other methods (at $z>1.96$, $p<0.025$) on all data sets. The cases where the performance of ADF-S is not significantly superior are marked with down arrow signs on top of the bars. We also get similar results for ADF-R. 

\begin{figure}[ht!]
\centering
  \setlength{\belowcaptionskip}{0pt}
	\setlength{\abovecaptionskip}{0pt}	
	        \includegraphics[width=1\linewidth]{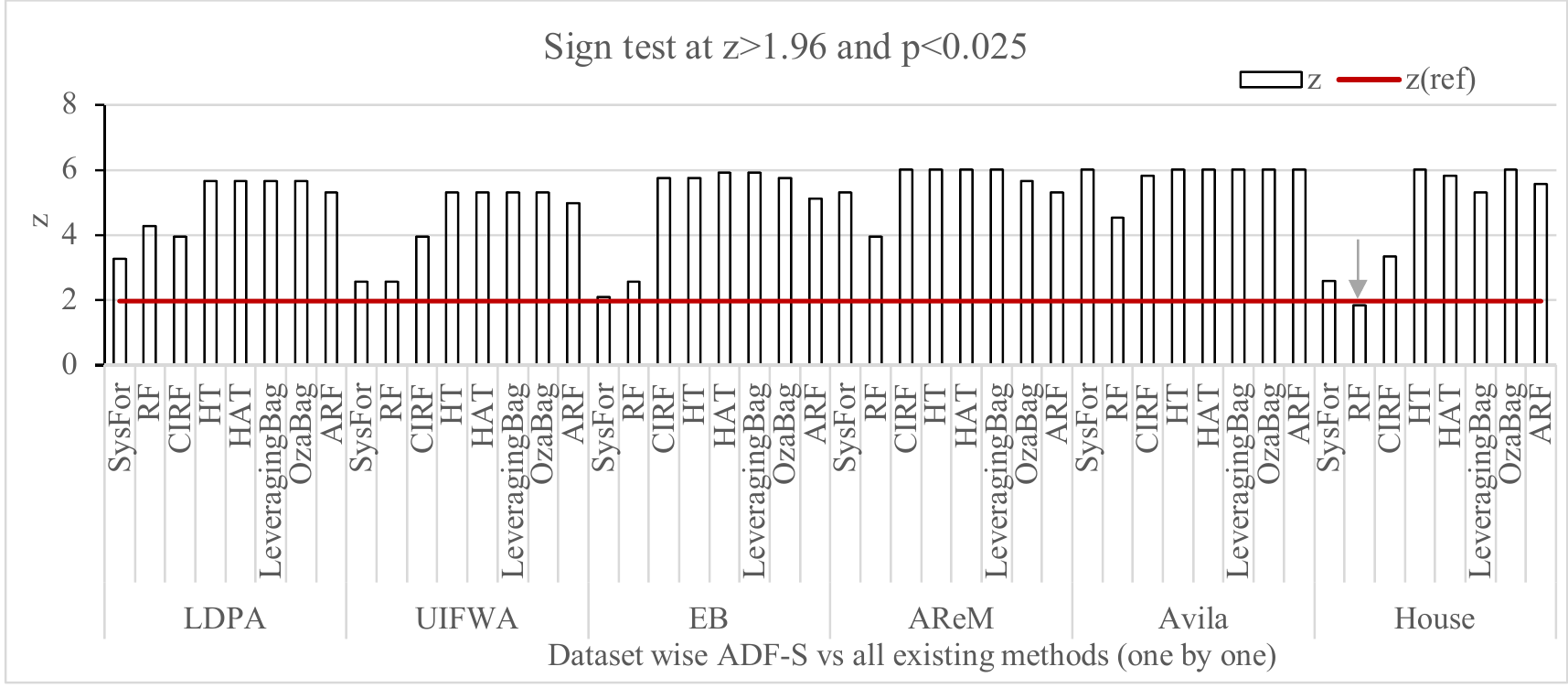}
	    \caption{Statistical analysis based on sign test on all datasets.}
	    \label{fig:sign_all_adf_s}
\end{figure}

We also carry out the Nemenyi test on the results of ADF-S with other methods (one by one) at the right-tailed by considering significance level $\alpha=0.025$ (i.e. 97.5\% significance level). The Nemenyi test results indicate that both ADF-S and ADF-R perform significantly better than the other techniques (at $p<0.025$) on all data sets.

\subsection{Experimentation in Handling Big data using ADF}
\label{Experimental_bigdata}

ADF can handle big data where the data can be divided into batches. For this task, we consider two big data sets, namely LDPA and UIFWA (see Table~\ref{tab:Datasets}). 

For the LDPA data set, we create a training data set having 80\% of the records that are selected randomly, and a test data set having the remaining records. To calculate the accuracy of the benchmark method such as Sysfor~\cite{islam2011knowledge}, we build a decision tree by applying Sysfor~\cite{islam2011knowledge} on the training data set and the tree is used to classify the test records. The accuracy of the benchmark method on the LDPA data set is 65.86\% as shown in Table~\ref{tab:expbigdata}. Besides, for ADF, we create 34 equal batch data sets where the records are chosen randomly from the training data set of the LDPA data set. We build a decision tree by applying Sysfor on Batch 1. The tree is then updated incrementally for the remaining 33 batches. The final tree is used to classify the test records. The classification accuracy of ADF on the LDPA data set is 61.88\% (see  Table~\ref{tab:expbigdata}). We also calculate the accuracy of CIRF on the LDPA data set, which is 53.77\%. 

Similarly, for the UIFWA data set, the accuracies of the benchmark method, CIRF and ADF are 45.90\%, 35.77\% and 40.91\%, respectively, as reported in Table~\ref{tab:expbigdata}. For both data sets, the performance of ADF is better than CIRF and is closer to the benchmark method. The experimental results on two data sets indicate the effectiveness of ADF over CIRF for handling big data.

\begin{table}[ht!]
	\small
	\centering
	\caption{Classification accuracy of ADF, CRIF and Benchmark methods on two big data sets.}
	\renewcommand\tabcolsep{3pt} 
	\begin{tabular}{lccc|lccc}
	\toprule
		Data set&Benchmark (SysFor)& CIRF & ADF&Data set&Benchmark (SysFor)& CIRF & ADF\\
		\midrule
		LDPA&65.86\%&53.77\%&61.88\%&UIFWA&45.09\%&35.77\%&40.91\%\\
	\bottomrule
	\end{tabular}
	\label{tab:expbigdata}
\end{table}

\section{Conclusion and future work}
\label{Conclusion}

This paper introduced ADF, an incremental machine learning framework, which produces a decision forest to classify new data. ADF can learn from new data that arrive as batches over time and the data can have new classes. As accurate classification is the primary goal in real applications, we argue that an incremental decision forest algorithm can achieve high accuracy based on three factors: identification of the best split, determination of trees that need to be repaired, and identification and management of concept drift. Based on our two novel theorems (see Theorem~\ref{subbox_nonoverlap_theorem} and Theorem~\ref{subbox_overlap_theorem}), we introduce a novel splitting strategy called iSAT (see Section~\ref{isat}) which can find the best split for new batch data. We also introduce a repairable strategy (see section~\ref{dtrs}) to find trees that need to be repaired. Moreover, we build a set of forests (see Section~\ref{justify_pf_af_tf}) to identify and handle concept drift and preserve previously acquired knowledge.

The effectiveness of ADF is also reflected in the experimental results. In the ADF framework, we build decision forests by using one of the three techniques: RF~\cite{breiman2001random}, HT~\cite{domingos2000mining} and SysFor~\cite{islam2011knowledge}, and thereby obtain three variants called ADF-R, ADF-H and ADF-S, respectively. We evaluate ADF variants on five publicly available natural data sets~\cite{Frank+Asuncion:2010} and one synthetic data set by comparing its performance with the performance of eight state-of-the-art techniques including HT~\cite{domingos2000mining}, CIRF~\cite{hu2018novel}, and ARF~\cite{gomes2017adaptive}. From Table~\ref{tab:overall_avg_accuracy}, we can see that in all data sets ADF-R and ADF-S outperform the other techniques in terms of classification accuracy, while the variants require a comparable execution time. Statistical sign test and Nemenyi test results (see Fig.~\ref{fig:sign_all_adf_s}) indicate that ADF-R and ADF-S perform significantly better than the other methods, except one case, at $z>1.96$, $p<0.025$ on all data sets.

Our initial experimentation on two data sets (see Table~\ref{tab:expbigdata}) indicates that ADF is also applicable to big data applications where the data can be divided into batches. In future work, we plan to explore the applicability of ADF for the non-dividable big data applications. 

{\footnotesize
\bibliography{gea_pd_ref}

}
\end{document}